\documentclass[11pt, a4paper, logo, copyright]{googledeepmind}

\pdftrailerid{redacted}

\usepackage{natbib}
\setcitestyle{numbers,open={[},close={]}}
\usepackage[utf8]{inputenc} 
\usepackage[T1]{fontenc}    
\usepackage{hyperref}       
\usepackage{url}            
\usepackage{booktabs}       
\usepackage{kantlipsum, lipsum}
\usepackage{xcolor}
\definecolor{darkblue}{rgb}{0, 0, 0.5}
\hypersetup{colorlinks=true, citecolor=darkblue, linkcolor=darkblue, urlcolor=darkblue}
\usepackage{etoc}
\usepackage{dsfont}
\usepackage{cleveref}
\usepackage{gdm-colors}

\usepackage{amssymb,amsmath,amsthm,bbm}
\usepackage{verbatim,float,url,dsfont}
\usepackage{graphicx,subfigure,psfrag}
\usepackage{algorithm, algpseudocode}
\usepackage{algorithmicx}
\usepackage{bibentry}
\algtext*{EndWhile}
\algtext*{EndIf}
\algtext*{EndFor}
\algtext*{EndProcedure}
\usepackage{mathtools,enumitem}
\usepackage{multirow}
\usepackage{ragged2e}
\usepackage{xr-hyper}
\usepackage{array}
\usepackage{stmaryrd,scalerel}
\usepackage[dvipsnames]{xcolor}
\usepackage{tikz}
\usepackage[utf8]{inputenc} 
\usepackage{accents}
\usepackage{xcolor}
\newcommand{\bunderline}[1]{\underaccent{\bar}{#1}}

\usepackage{booktabs}       
\usepackage{nicefrac}         
\usepackage{microtype}      
\usepackage{fancyhdr}

\newtheoremstyle{mytheoremstyle}
{.5em}  
{.5em}  
{\itshape} 
{}     
{\bfseries} 
{.}    
{.5em} 
{}     
\theoremstyle{mytheoremstyle}
\newtheorem{theorem}{Theorem}
\newtheorem{lemma}[theorem]{Lemma}
\newtheorem{corollary}[theorem]{Corollary}
\newtheorem{proposition}[theorem]{Proposition}
\newtheorem{remark}[theorem]{Remark}

\newtheorem{assumption}[theorem]{Assumption}

\def\E{\mathbb{E}}

\usepackage{mathtools}

\usepackage{mathtools}
\mathtoolsset{showonlyrefs}\usepackage{afterpage}

\usepackage{epigraph}

\usepackage[most,skins,theorems]{tcolorbox}

\tcbset{
  aibox/.style={
    width=\linewidth,
    top=6pt,
    bottom=0pt,
    colback=blue!6!white,
    colframe=black,
    colbacktitle=black,
    enhanced,
    center,
    attach boxed title to top left={yshift=-0.1in,xshift=0.15in},
    boxed title style={boxrule=0pt,colframe=white,},
  }
}
\newtcolorbox{AIbox}[2][]{aibox,title=#2,#1}

\tcbuselibrary{listings, breakable}
\tcbset{
    graybox/.style={
        colback=gray!20,  
        colframe=black,   
        boxrule=1pt,      
        arc=4mm,          
        width=\textwidth, 
        before=\vspace{10pt}, 
        after=\vspace{10pt},  
        center             
    }
}

\title{CollabEval: Statistically Efficient Collaborative Model Evaluation via Matrix Completion}

\author[1]{Adam Fisch}
\author[1]{Daniel Deutsch}
\author[1]{Joshua Maynez}
\author[2]{Alekh Agarwal}
\author[2]{Jonathan Berant}
\author[1]{William Cohen}
\author[2]{Amir Globerson}
\author[1]{Jacob Eisenstein}

\affil[1]{Google DeepMind}
\affil[2]{Google Research}

\begin{abstract}
Evaluating generative AI models is a routine, but resource-intensive, process that is conducted over and over again during the course of model development. In this work, we propose Collaborative Evaluation (CollabEval), a simple, effective, and principled method for exploiting dependencies between historical runs of different models on the same tasks to improve statistical efficiency. Specifically, our approach treats model evaluation as a matrix completion problem over an $M \times N$ matrix of evaluation scores, where $M$ is the total number of models and $N$ is the total number of evaluation prompts. We assume that a subset of these $M$ models are targeted for evaluation. For these target models only a small fraction, $p$, of prompts has been annotated with evaluation scores. Leveraging recent results in prediction-powered inference, we build a low-rank approximation of the score matrix, and use the reconstructed values as control variates in a manner that guarantees unbiased estimates of the true evaluation metric mean, in addition to statistically valid confidence intervals. Empirically, across a wide range of datasets, models, and sparsity levels $p$, we find that CollabEval substantially reduces the mean confidence interval size, and the mean squared error of the point estimate, compared to baseline methods at the same annotation budget.\looseness=-1

\end{abstract}

\usepackage[parfill]{parskip}
\begin{document}
\maketitle

\section{Introduction}
\label{sec:intro}
Evaluating large-scale generative AI models typically involves gathering prompts from an input distribution, generating model outputs, and scoring all outputs with a reliable rater (e.g., an LLM-as-a-judge or human annotator). Though straightforward in principle, obtaining reliable results requires gathering a non-trivial number of data points, which can make evaluation costly. While recent work has explored the viability of using more efficient automatic metrics and autoraters as  proxies for expensive human annotations, even the \emph{inference} step for most large AI models carries substantial overhead (e.g., especially when long tool-call chains are involved), motivating methods that bypass generation entirely.\looseness=-1

Evaluation costs compound when many evaluations are conducted simultaneously or continuously over long periods of time---for instance, when comparing fine-tuned model variants, or when monitoring for regressions against new checkpoints. The systems under comparison often share architectures, training data, or system prompts, so their performances are correlated. Standard methods, however, usually treat each of these evaluations as a distinct estimation problem, and fail to leverage this correlational structure. This leads to a suboptimal use of the available evaluation budget.\looseness=-1

\begin{figure}[!t]
    \centering
    \includegraphics[width=.9\linewidth]{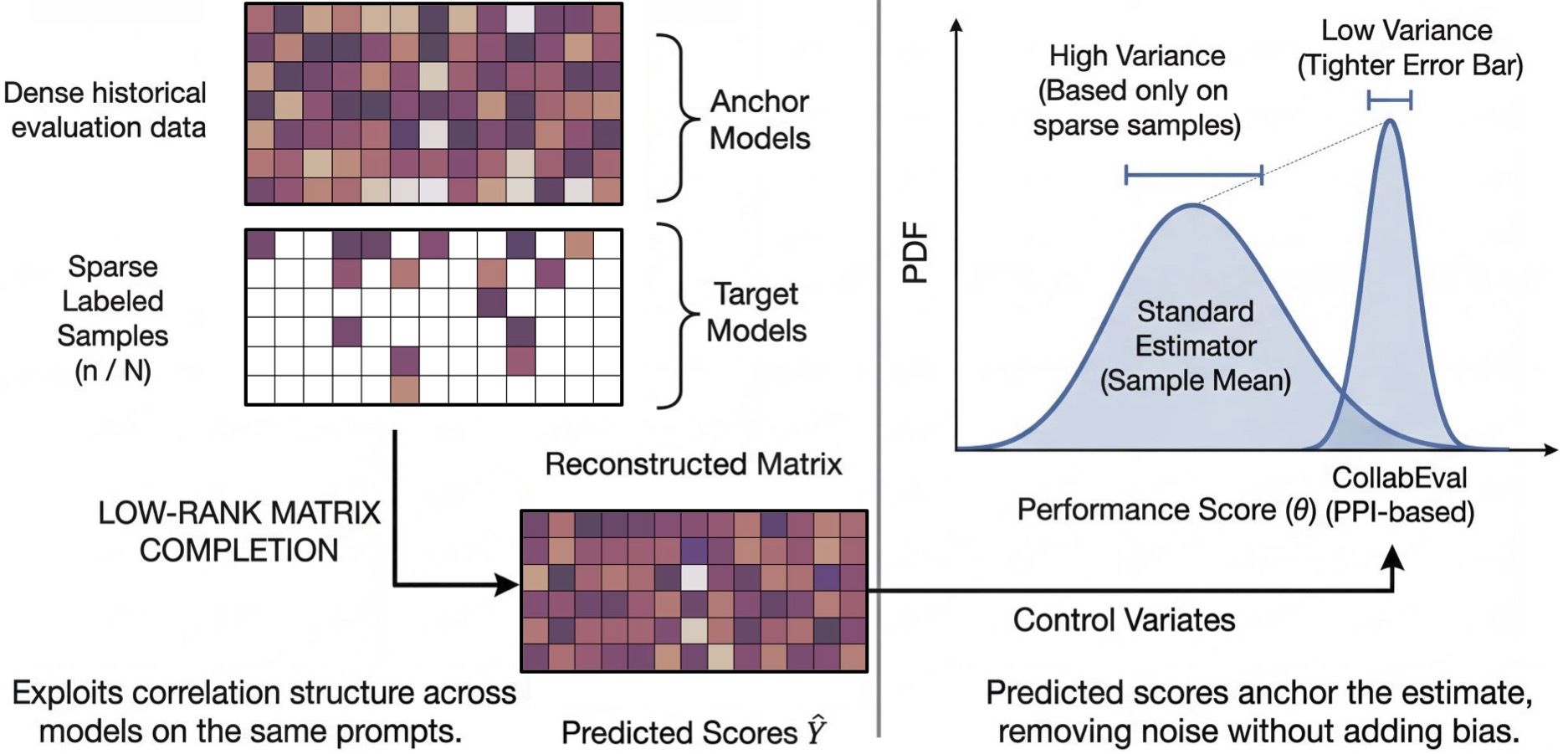}
    \vspace{-3pt}
    \caption{An illustration of the core concepts underlying CollabEval. CollabEval takes in an $M \times N$ matrix of evaluation scores, divided into dense historical data from "Anchor Models" and sparse labeled samples from new "Target Models". By applying a type of cross-fold matrix completion (Section~\ref{sec:cross_fold_matrix}), we reconstruct the matrix to predict missing scores. These predictions  are then used as control variates within a framework based on cross-prediction-powered inference (CrossPPI; see \citet{zrnic2025cross}) to obtain more accurate estimates with tighter, valid confidence intervals.\looseness=-1}
    \label{fig:intro}
   \vspace{-3pt}
\end{figure}

In this work, we propose \emph{Collaborative Evaluation} (CollabEval): a lightweight approach that exploits dependencies between model output scores to increase the \emph{effective sample size} of an evaluation, thereby \emph{reducing the set of evaluated prompts} while preserving accuracy and confidence. Unlike methods that reduce \emph{labeling costs} via cheaper autoraters but still incur \emph{generation costs}~\cite{boyeau2024autoeval, fisch2024stratified, angelopoulos2025costoptimalactiveaimodel}, skipped prompts in our setting incur nearly \emph{zero cost}---requiring neither inference nor scoring. Our method has two parts. First, we reframe evaluation with partially observed scores as a collaborative filtering task: we represent the full set of models $\times$ prompts $\rightarrow$ scores as an $M \times N$ matrix $S$, where only a fraction $p$ of entries are observed for the target models we want to evaluate (a subset of the $M$ total models; other non-target models may be historical). We use matrix completion to try to impute the unobserved scores, transferring knowledge across related evaluations. Then  we use these imputed scores as control variates to reduce variance, while guaranteeing unbiased estimates with asymptotically valid confidence intervals that are simple to compute.
Importantly, the theoretical validity of these intervals relies entirely on the control variate framework; CollabEval maintains valid  coverage even if the score matrix lacks low-rank structure and the attempted imputation is noisy.\looseness=-1 

We develop CollabEval by combining collaborative filtering techniques with prediction-powered inference \cite{angelopoulos2023prediction}---specifically cross-prediction-powered inference~\cite{zrnic2025cross}---and introduce technical extensions for rigorously estimating both \emph{pointwise} metrics (e.g., the absolute accuracy of model $A$) and \emph{pairwise} differences (e.g., the accuracy delta between models $A$ and $B$). Notably, unlike prior work such as item response theory (\Cref{sec:related}), our framework places few restrictions on the evaluation metric, the matrix completion algorithm, or the composition or size of the \emph{anchor model} set (the models for which we have dense annotations).
Furthermore, when combined with power tuning~\cite{angelopoulos2023ppi++},  CollabEval  is guaranteed to have an asymptotic variance less than or equal to classical estimation (Proposition~\ref{prop:power_tuning_distribution}). This variance reduction scales with the squared covariance between the true target model scores and the imputed predictions from matrix completion; if imputed values are uncorrelated with true values, the estimator gracefully falls back to classical estimation. The magnitude of the variance reduction in realistic practical scenarios is an empirical question. To that end, we evaluate CollabEval on five diverse text generation benchmarks across a variety of real-world experimental settings, and show that it can \textbf{reduce confidence interval widths by up to 30\%} compared to classical inference.

\begin{AIbox}{Summary of main contributions}
Model evaluation suffers from high variance when the sample size is small. However, different models are also often closely correlated. CollabEval mitigates evaluation variance by reframing model evaluation as a matrix completion problem---using results from other, historical, models to impute scores for skipped evaluation prompts from a (potentially) much larger evaluation set. We derive control variates from these predictions via cross-prediction-powered inference~\cite{zrnic2025cross}, yielding lower mean squared evaluation error and tighter confidence intervals.\looseness=-1
\end{AIbox}

\section{Related work}
\label{sec:related}

CollabEval builds on recent work in prediction-powered inference~\citep{angelopoulos2023prediction, angelopoulos2023ppi++, zrnic2025cross}, which can reduce the cost of model evaluation by combining cheap-and-noisy "weak" labelers, such as LLM-as-judge, with more expensive and accurate annotations, such as human ratings~\citep{boyeau2024autoeval,gligoric2024can,fisch2024stratified,angelopoulos2025costoptimalactiveaimodel}. This line of work has focused on reducing the cost of \emph{labeling} model outputs, but in all cases still requires paying the computational price to \emph{generate} the outputs. Furthermore, at least in framing, these approaches have always required an explicit, standalone label predictor. Instead, the key insight in our work  is to use matrix completion to produce noisy annotations without the cost of generation---and effectively use the historical scores of other models as the only "features" needed to predict a target model's performance (though we note that additional covariates for each prompt / model pair could also be readily incorporated into our framework as part of a \emph{contextual} matrix completion algorithm).\looseness=-1

The use of matrix reconstruction is motivated by growing evidence that model evaluation scores exhibit strong low-rank structure---not only across benchmarks, where a handful of evaluations can predict the rest~\cite{ye-etal-2023-predictable, zhang-etal-2024-collaborative, papailiopoulos2026benchpress}, but also at the \emph{item level}, where individual prompt scores are highly correlated across models~\cite{zhou2025on}. CollabEval exploits this finer-grained redundancy to reduce the number of items evaluated \emph{within} each benchmark, while still providing rigorous statistical guarantees---and puts  this general collaborative filtering style of approach to evaluation  on solid  footing. Prior work on  evaluation has also leveraged such cross-model correlations, specifically through Item Response Theory~\cite[IRT;][]{ martinez2019item, rodriguez-etal-2021-evaluation, pmlr-v235-maia-polo24a, kipnis2025metabench, hofmann2025fluid, polo2026richinsightscheapsignals}, which factors a matrix of items and scores to estimate an "ability" parameter per model. CollabEval is based on the same intuitions, but rather than estimating a logit-space ability parameter, it enables accurate and precise estimates of any evaluation characteristic of interest. Importantly, matrix completion can be seen as a generalization of IRT, at least in terms of predictive power: while a valid IRT model implies a very low rank score (log odds) matrix, the converse does not necessarily hold, allowing our framework to capture model-item interactions that IRT might miss. Of course, IRT models can also be used as drop-in substitutes for the matrix completion algorithm in CollabEval (and we report results for 2PL IRT models in our experiments for binary tasks).\looseness=-1

Finally, we note that another class of related work is based on coresets~\cite[see, e.g., Anchor Points;][]{vivek-etal-2024-anchor}, in which evaluation data is downsampled to a few key datapoints. In contrast, CollabEval leverages the unlabeled data rather than throwing it away---and furthermore does so in a manner that yields \emph{unbiased} estimates and \emph{statistically valid confidence intervals} without strong assumptions about the underlying data. Empirically, CollabEval also exhibits greater empirical stability across tasks.
\section{Background}
\label{sec:background}

 Collaborative Evaluation is built upon the statistical framework of prediction-powered inference \cite[PPI;][]{angelopoulos2023prediction, angelopoulos2023ppi++}. PPI is designed to efficiently estimate population parameters (in our case, the mean model performance, $\theta = \E[Y]$) in settings where acquiring the true labels $Y$ is expensive or difficult. The standard approach of calculating a sample mean from a small set of $n$ annotated data points is unbiased, but often suffers from high variance, yielding confidence intervals that are too wide for meaningful decision-making.
PPI provides a principled way to reduce this variance by leveraging a much larger, cheaply-obtained set of $N + n$ ``predicted'' labels $\hat{Y}$ on both the original labeled set of $n$ examples, and an additional set of $N$ unlabeled examples---though these predictions may be systematically biased. PPI then combines the high-cost, unbiased data with the low-cost, biased data using a ``rectified'' mean, as shown in Equation 
\eqref{eq:ppi_full}.\looseness=-1
\begin{equation}
\label{eq:ppi_full}
\hat{\theta}_{\textrm{ppi}} = \underbrace{\frac{1}{N} \sum_{i=1}^N \hat{Y}_{i}}_{\text{prediction mean}} ~+~ \underbrace{\frac{1}{n} \sum_{i=N + 1}^{N + n} Y_{i} - \hat{Y}_{i}}_{\text{prediction bias rectifier}},
\end{equation}
where the random variables $\{\hat{Y}_i\}_{i=1}^{N+n}$ and $\{Y_i\}_{i=N+1}^{N+n}$ are respectively assumed to be independent and identically distributed (iid).
Notice that this construction then has two useful properties. First, the PPI estimator, $\hat{\theta}_{\textrm{ppi}}$, is unbiased (i.e., $\mathbb{E}[\hat{\theta}_{\textrm{ppi}}] = \theta$), regardless of the accuracy of $\hat{Y}$ with respect to the true $Y$. In expectation, the rectifier term cancels out the  bias of the predicted mean. Second, when $N \gg n$, the variance of $\hat{\theta}_{\textrm{ppi}}$ no longer primarily depends on the  variance of the true scores, $\text{Var}(Y)$, but rather on the variance of the prediction \emph{residuals}, $\text{Var}(Y - \hat{Y})$. If the predicted scores $\hat{Y}$ are strongly correlated with the true scores $Y$, the residual variance will be small. This will then result in a more precise estimate of $\theta$, and significantly tighter confidence intervals, for the same small labeling budget $n$.

\textbf{Why bias correction is important.} To take a moment to understand why the rectification step in PPI is important, consider the danger of a naive combined estimator. If one were to simply use the  predictions $\hat{Y}$ as drop-in replacements for the unobserved labels, the resulting estimator 
\begin{equation}
\hat{\theta}_{naive} = \frac{1}{N + n}\left(\sum_{i=1}^N \hat{Y}_i + \sum_{i=N+1}^{N+n} Y_i\right)
\end{equation}
could dramatically reduce variance. However, this comes at the cost of introducing a potentially severe bias whenever the predictor systematically under- or over-estimates the true performance,  and cause the mean squared error to be dominated by $(\frac{N}{N + n})^2(\mathbb{E}[\hat{Y}] - \mathbb{E}[Y])^2$. PPI sidesteps this by instead using the predictions as a control variate~\cite{Ripley87} to drive down the variance, without incurring bias.\looseness=-1

\subsection{Cross-prediction-powered inference}
\label{sec:cross_ppi}
Conventional PPI relies on the availability of an independent, pre-existing model to generate the predicted values $\hat{Y}$. If this is not the case, the simplest thing to do is to split the $n$ labeled data points into a proper training and test set, where the training set is used to train a predictor that yields $\hat{Y}$ (e.g., given features $X$ for label $Y$), and the test set is used for estimating $\theta$ per Equation~\eqref{eq:ppi_full}. This is not, however, the most \emph{efficient} way to use the labeled data points, when $n$ is small.

Cross-prediction-powered inference \cite[CrossPPI;][]{zrnic2025cross} adapts PPI to  use the entire labeled dataset \emph{both} for training the predictor that gives $\hat{Y}$ \emph{and} for estimating $\theta$. CrossPPI achieves this through a variation of cross-validation. First, the labeled data points are randomly partitioned into $K$ disjoint, exchangeable folds, $\{\mathcal{I}_1, \ldots, \mathcal{I}_K\}$.
Let $\hat{Y}^{(k)}$ denote the predictions generated from a predictor that is trained on all folds $j \neq k$. These ``cross-predictions'' are then simply aggregated across all folds:
\begin{equation}
\label{eq:cross_ppi}
\hat{\theta}_{\textrm{crossppi}} = \underbrace{{\frac{1}{NK} \sum_{i=1}^N \sum_{k=1}^K\hat{Y}_i^{(k)}\vphantom{\sum_{i \in \mathcal{I}_k}}}}_{\text{cross-prediction mean}} + \underbrace{\frac{1}{n} \sum_{k=1}^K\sum_{i \in \mathcal{I}_k} Y_i - \hat{Y}_i^{(k)}.}_{\text{cross-prediction bias rectifier}}
\end{equation}
Under mild   assumptions, \citet{zrnic2025cross} show that as $n$ and $N$ grow, $\hat{\theta}_\mathrm{crossppi}$ is unbiased, low variance, and normally distributed: a fact useful for deriving confidence intervals.\looseness=-1

In the next section, we adapt CrossPPI to matrix completion using an $M \times N$ matrix of partially observed scores for $M$ models across $N$ inputs. This yields statistically efficient estimates with asymptotically valid confidence intervals for both individual model means and pairwise differences.
\section{Collaborative Model Evaluation via Matrix Completion}
\label{sec:method}
The main practical challenge of PPI-based methods, and \emph{the core problem that CollabEval addresses}, is how to efficiently acquire a high-quality set of predictions $\hat{Y}$ in the first place. As outlined in Section~\ref{sec:intro}, our main hypothesis is that for model evaluation in particular, the evaluation scores are highly correlated across related models: observing a subset of  scores from these other models is sufficient to construct high-quality predictions $\hat{Y}$ for a new model of interest, for prompts that we do not even have to run the new model on at all. 

At a high-level,  CollabEval is fairly simple: we use matrix completion to impute missing scores, and then correct for the bias that the noisy predictions may have. The following sections  carefully setup this up. In Section~\ref{sec:setting} we formally define the setting. Then, in Section~\ref{sec:cross_fold_matrix}, we describe the specifics of our cross-fold matrix completion algorithm, and resulting estimator. Finally, in Section~\ref{sec:interval} we derive confidence intervals with rigorous statistical guarantees (however, we defer the formal statements and proofs of the guarantees to Appendices~\ref{app:theory} and \ref{app:proofs}). See also Appendix~\ref{sec:implementation} for full implementation details.\looseness=-1

\subsection{A motivating example}
\label{sec:example}
To start, consider a scenario where we want to evaluate a target model $T$ across a diverse set of prompts, while leveraging historical evaluations from two anchor models, $A_1$ and $A_2$. We can assume that $T$'s performance correlates differently  with that of $A_1$ and $A_2$, depending on the type of input. For example, $T$ might be comparable to $A_1$ on "Type 1" prompts (e.g., math reasoning), but then more like $A_2$ on "Type 2" prompts (e.g., creative writing). Relying on either $A_1$ or $A_2$ alone as a global control variate would be suboptimal. Ideally, an oracle control variate would dynamically route predictions (i.e., use $A_1$ for  Type 1 prompts and $A_2$ for  Type 2 prompts) to obtain small residual variance across the entire dataset.\looseness=-1

The general motivation for using matrix completion techniques here is to naturally approximate this sort of "control variate routing" behavior, or potentially something better. For example, in the low-rank matrix factorization approach we take (see Section~\ref{sec:cross_fold_matrix}), observing just a few scores from model $T$ allows us to project it into a shared latent space alongside the other anchor models and the prompts. The predicted score $\hat{Y}$ for model $T$ for any unobserved prompt is a linear combination of the anchor models' scores on that same prompt, where the weighting  depends on the learned latent factors.\looseness=-1 

\subsection{Setting}
\label{sec:setting}
Let $S \in \mathbb{R}^{M \times N}$ be a matrix of evaluation scores, where $M$ is the number of models, $N$ is the total number of input prompts, and $S_{ij}$ is the score of model $i$'s response to prompt $j$. Assuming the $N$ prompts are sampled i.i.d. from a distribution of interest, each $S_{ij}$ is a random variable.
We use $S_i$ to represent the marginal score distribution for model $i$, meaning $S_{i} \overset{d}{=} S_{ij}$ for all $j \in \{1, \ldots, N\}$. 

We partition the models into two disjoint sets: \emph{anchor models} $\mathcal{A} \subset \{1, \ldots, M\}$, for which we have historical data (having evaluated many entries in an anchor row $i \in \mathcal{A}$), and \emph{target models} $\mathcal{T} \subseteq \{1, \ldots, M\}$, which are new models with initially unobserved scores.\footnote{In a standard real-world evaluation pipeline, we might assume that the anchor model set is periodically refreshed  to ensure the historical data remains relevant to the current evaluation targets.}

For each target model $i$, our goal is to estimate  $\theta_i = \mathbb{E}[S_i]$. Since model evaluations are ultimately often used to do model comparisons, we also estimate the mean performance \emph{difference} between the target model $i$ and any other model $j \neq i$, defined as $\Delta_{ij} = \mathbb{E}[S_i - S_j]$. The comparison model $j$ can be either a target or anchor model. To reduce the cost of evaluating all $N$ prompts for the target models, we observe only a small subset of $n_i$ independently selected indices  from $\{1, \dots, N\}$ for each target model $i \in \mathcal{T}$.\footnote{We slightly abuse the prior notation from Section~\ref{sec:background}:  $n_i \leq N$ now denotes a \emph{subset} of the full dataset.\looseness=-1}
We use $\Omega \in \{0, 1\}^{M \times N}$ to denote the binary matrix of observed indices:
$$\Omega_{ij} = 
\begin{cases} 
1 & \text{if model $i$'s evaluation score for prompt $j$ is observed,} \\
0 & \text{otherwise.}
\end{cases}$$
We assume that $\Omega$ is independent from the prompts. For the entries $(i,j)$ where $\Omega_{ij} = 0$, $S_{ij}$ is imputed via a procedure that we call \emph{cross-fold matrix completion}, which we describe next.

\subsection{Cross-fold matrix completion}
\label{sec:cross_fold_matrix}
Let $\mathcal{M} \colon \mathbb{R}^{M \times N} \times \{0, 1\}^{M \times N} \to \mathbb{R}^{M \times N}$ be a generic matrix completion algorithm that accepts a partially observed matrix $S$ and a binary  mask $\Omega$, and returns a matrix $\hat{S}$ where all missing entries $(i, j)$ for which $\Omega_{ij} = 0$ have been imputed. We use an iterative algorithm based on SVD \cite[see, e.g.,][]{troyanskaya2001missing, mazumder2010spectral} that recursively updates missing entries  in the matrix $S$ by using low-rank approximations for $\hat{S}$ (see Appendix~\ref{sec:implementation} for  details, and comparisons to other alternatives). It is also worth noting that the computational overhead of this matrix completion is small for the sizes of evaluation data matrices that we consider, even when repeated multiple times. As the imputed scores in $\hat{S}$ depend on the observed scores in the original $S$, we use a $K$-fold cross-completion strategy  based on CrossPPI in Section~\ref{sec:cross_ppi}. This will allow us to estimate the bias in our imputed evaluation scores, and correct for it.\looseness=-1

First, we identify the set of active prompt indices $\mathcal{U} = \bigcup_{i \in \mathcal{T}} \mathcal{J}_i$, which is the set of prompts for which at least one target model has an observation. We randomly partition $\mathcal{U}$ into $K$ disjoint folds, $\{\mathcal{U}_1, \dots, \mathcal{U}_K\}$.
For each target model $i$, this splits its observed indices $\mathcal{J}_i$ into $K$ folds, $\{\mathcal{J}_i^{(1)}, \dots, \mathcal{J}_i^{(K)}\}$, where $\mathcal{J}_i^{(k)} = \mathcal{J}_i \cap \mathcal{U}_k$. We then construct a sequence of $K$ cross-prediction masks, $\{\Omega^{(1)}, \dots, \Omega^{(K)}\}$, where each mask $\Omega^{(k)}$ holds out the $k$-th fold of labeled prompts across all target models. Specifically, for each fold $k$, we set $\Omega^{(k)}_{ij} = 0$ for all $i \in \mathcal{T}$ and all $j \in \mathcal{U}_k$, while keeping $\Omega^{(k)}_{ij} = \Omega_{ij}$ otherwise.

Applying the completion algorithm to each mask yields a corresponding set of prediction matrices, $\hat{S}^{(k)} = \mathcal{M}(S, \Omega^{(k)})$. We then aggregate these  matrices into a single prediction matrix $\hat{Y}$ using the held-out prediction for observed entries and the cross-fold average for unobserved entries:
\begin{equation}
\label{eq:cross_fold_matrix}
\hat{Y}_{ij} = 
\left\{
\begin{array}{ll@{\quad}l}
\hat{S}^{(k)}_{ij} & \text{if } j \in \mathcal{J}_i^{(k)} \text{ and } i \in \mathcal{T} & \text{(cross-fold prediction independent of } S_{ij}\text{)} \\
\displaystyle \frac{1}{K} \sum_{k=1}^K \hat{S}^{(k)}_{ij} & \text{otherwise} & \text{(average cross-fold prediction)}
\end{array}
\right.
\end{equation}
Finally, we estimate the performance  $\hat{\theta}_i$ using a slightly rewritten form of CrossPPI,\looseness=-1
\begin{equation}
\label{eq:cross_matrix_ppi}
\hat{\theta}_i = \frac{1}{N} \sum_{j=1}^N \hat{Y}_{ij} + \frac{1}{|\mathcal{J}_i|} \sum_{j \in \mathcal{J}_i} (S_{ij} - \hat{Y}_{ij})
= \frac{1}{|\mathcal{J}_i|} \sum_{j \in \mathcal{J}_i} S_{ij} - \left(\frac{1}{|\mathcal{J}_i|} \sum_{j \in \mathcal{J}_i} \hat{Y}_{ij} - \frac{1}{N}\sum_{j=1}^N \hat{Y}_{ij}\right),
\end{equation}
and the difference in performance between models $i$ and $j$ as $\hat{\Delta}_{ij} = \hat{\theta}_i - \hat{\theta}_j$.

Note that Equation~\eqref{eq:cross_matrix_ppi} calculates the mean prediction using all $N$ items (versus only the $N -n$ unlabeled items where $\Omega_{ij} = 0$ as standard PPI would do), and rewrites the estimator in a classical control variate form~\citep{Ripley87}. We use this form for convenience (it can be shown that the two versions have equivalent asymptotic variance when $\hat{Y}_{ij}$ is scaled by an optimal constant of choice).

\subsection{Confidence interval estimation}
\label{sec:interval}

When evaluating models, we typically want to meaningfully quantify the uncertainty in our estimates. CollabEval allows us to easily compute confidence intervals by relying on plug-in empirical estimates of the  covariance between the model evaluation scores.

Let $n_i = |\mathcal{J}_i|$ be the number of observed labels for model $i$, and $n_{ij} = |\mathcal{J}_i \cap \mathcal{J}_j|$ be the number of overlapping observed labels between models $i$ and $j$. We then estimate the covariance matrix $\hat{\Sigma}$ for the joint vector of CollabEval predictions $\sqrt{N}\hat{\theta} \in \mathbb{R}^M$  (which, in Appendix~\ref{app:theory},  we show converges to the true asymptotic covariance matrix $\Sigma$ under mild assumptions) using:\looseness=-1
%
\begin{equation}
\label{eq:empirical_covariance}
\hat{\Sigma}_{ij} = \frac{N n_{ij}}{n_i n_j} \widehat{\mathrm{Cov}}(S_i, S_j) + \left(\frac{N n_{ij}}{n_i n_j} - 1\right) \left[ \widehat{\mathrm{Cov}}(\hat{Y}_i, \hat{Y}_j) - \widehat{\mathrm{Cov}}(\hat{Y}_i, S_j) - \widehat{\mathrm{Cov}}(S_i, \hat{Y}_j) \right],
\end{equation}
where the empirical covariances $\widehat{\mathrm{Cov}}(\cdot, \cdot)$ involving true scores $S$ are calculated over the intersecting observed sets (e.g., $j \in \mathcal{J}_i \cap \mathcal{J}_j$), and the empirical covariances of the  predicted scores $\hat{Y}$ are calculated over the all $N$ prompts.
We can then construct confidence intervals with $(1-\alpha)$  coverage for both the individual model means and the pairwise differences as follows:\looseness=-1
\begin{equation}
\mathcal{C}_\alpha(\hat{\theta}_i) = \left[ \hat{\theta}_i \pm z_{1-\alpha/2}\sqrt{\frac{\hat{\Sigma}_{ii}}{N}} \right] 
\quad\text{and}\quad 
\mathcal{C}_\alpha(\hat{\Delta}_{ij}) = \left[ \hat{\Delta}_{ij} \pm z_{1-\alpha/2}\sqrt{\frac{\hat{\Sigma}_{ii} + \hat{\Sigma}_{jj} - 2\hat{\Sigma}_{ij}}{N}} \right],
\end{equation}
where $z_{1-\alpha/2}$ is the standard normal quantile. Crucially, as long as the matrix completion algorithm satisfies a basic stability condition (i.e., its predictions converge as data grows), these intervals are statistically valid. We summarize this in the following informal theorem.\looseness=-1

\begin{theorem}[CI Coverage of CollabEval: Informal]
\label{thm:informal_coverage}
Under mild stability assumptions on the completion algorithm $\mathcal{M}$, as the total number of prompts $N$ and the labeled subsets $n_i$ grow large, the estimators $\hat{\theta}_i$ and $\hat{\Delta}_{ij}$ are asymptotically normally distributed. Consequently, the confidence intervals $\mathcal{C}_\alpha(\hat{\theta}_i)$ and $\mathcal{C}_\alpha(\hat{\Delta}_{ij})$ converge to valid $(1-\alpha)$ marginal coverage of the true parameters, $\theta_i$ and $\Delta_{ij}$.
\end{theorem}

A formal statement of Theorem~\ref{thm:informal_coverage} and other  related results are provided in Appendix~\ref{app:theory}.
We note that the empirical covariance estimates above can also be used to compute optimal "power tuning" weights \citep{angelopoulos2023ppi++}, a technique that scales $\hat{Y}_{ij}$ by a scalar $\lambda_i \in \mathbb{R}$ to ensure that the asymptotic variance (and resulting confidence interval width) is never worse than that of the classical sample mean. We provide a derivation of the optimal weights $\lambda^* \in \mathbb{R}^M$ to use for estimating both means and differences in means (they can differ) in Appendix~\ref{sec:powertuning}. In the remainder of the paper, we  use these scaled estimators by default.\looseness=-1

\section{Experiments}
\label{sec:experiments}

\subsection{Datasets}

The datasets we use cover a diverse range of settings and metrics. Specifically, we assess instruction-following capabilities using \textbf{AlpacaEval 2.0}~\cite{alpaca_eval}, multiple-choice multitask language understanding via \textbf{MMLU}~\cite{hendrycks2020measuring}, grounded generation and evidence-supported correctness with \textbf{Attributed Question Answering (AQA)}~\cite{bohnet2023attributed}, machine translation across 55 languages using \textbf{WMT24++}~\cite{deutsch-etal-2025-wmt24}, and repository-level patch generation via \textbf{SWE-bench}~\cite{jimenez2024swebench}.\footnote{In Appendix \ref{sec:singular_values} we show that the evaluation matrices for these tasks indeed empirically exhibit low effective rank. Across all the five datasets, the first 16 singular vectors (or fewer) explain over 50\% of the variance in the observed scores.}
For AlpacaEval 2.0 and MMLU, we designate specific target models based on recent leaderboard submissions while using older or remaining models as anchors; conversely, for AQA, WMT24++, and SWE-bench, we dynamically construct our target and anchor sets using leave-one-out or leave-two-out evaluation protocols over the available systems. See Appendix~\ref{sec:datasets} for additional details. In the main text, we report performance metrics averaged across all datasets, with error bars denoting the \textbf{min} and \textbf{max}  results.  Individual task results are in Appendix~\ref{app:additional_results}.\looseness=-1

\subsection{Sampling strategies}
\label{sec:sampling}
We consider two straightforward strategies for sampling the  observation matrix $\Omega$:
\begin{enumerate}[leftmargin=*]
    \item \textbf{Independent sampling (IID):} For each target model, the subset of observed indices is drawn uniformly at random. This can be helpful for the matrix completion stage of CollabEval: by observing a more diverse set of prompts across the target models, we can better estimate the underlying low-rank structure, especially in settings where the set of anchor models is small, or does not fully span the subspace of the new target models.\looseness=-1
    \item \textbf{Paired sampling (Paired):} Here we observe the exact same subset of prompts for each target model ($\mathcal{J}_i = \mathcal{J}_j$). This  can be very advantageous when estimating $\Delta_{ij}$, the {difference} in model performance, if the scores $S_{i}$ and $S_{j}$ are highly correlated in ways that the predictions $\hat{{Y}}_{i}$ and $\hat{{Y}}_j$ do not fully capture. While this is a good variance reduction technique, it does also limit the diversity of observed data available to the matrix completion algorithm, and makes the anchor models critical for being able to reconstruct the target rows.
\end{enumerate}

\subsection{Baselines}

We compare CollabEval against the \textbf{classic sample mean} baseline, which simply computes the empirical mean over the observed samples. We apply this baseline using both the independent (IID) and paired sampling strategies described in Section~\ref{sec:sampling}. 
We then also consider three baselines (the results of which are deferred to Appendix~\ref{app:additional_results}, along with several other additional experimental results):\looseness=-1
\begin{itemize}[leftmargin=*]
\item \textbf{Naive Imputation}: Here we  directly construct a dense evaluation matrix by imputing all missing entries with one step of matrix completion.
We then naively calculate the sample mean and variance without any corrections. While this aggressively reduces variance, it is vulnerable to systematic bias, and, in general, will not produce valid confidence intervals.\looseness=-1


\item \textbf{Anchor Points \cite{vivek-etal-2024-anchor}}: A coreset-based approach that selects a small, representative subset of prompts to evaluate based on historical anchor model scores.\footnote{We note that the original method  uses real-valued model confidence scores; because these are  unavailable for the tasks here, we adapt our implementation to use only the vector of raw evaluation scores, which matches the inputs to CollabEval.}  We compute the mean performance scores for new models over these selected points only, and skip the rest. This yields a point estimate, but does not natively provide valid confidence intervals (and like naive imputation, can also be biased).\looseness=-1

\item \textbf{PPI (Anchor Mean)}: A baseline that replaces our matrix completion step with a simple, untrained control variate. Specifically, we compute the mean score of all available anchor models for prompt $j$ and use it as the drop-in prediction for any target model's score.\looseness=-1
\end{itemize}

Appendix \ref{sec:mc_ablation} also compares CollabEval across four alternative, drop-in, matrix completion algorithms: IterativeSVD (default), a neural network, nuclear norm minimization, and a 2-PL IRT model.\looseness=-1

\subsection{Metrics}
We compute all metrics over $10k$ bootstrap trials, where we resample the $N$ base examples from each dataset with replacement, and set the target ground-truth parameters $\theta$ and $\Delta$ to be the averages over the $N$ base examples. Our evaluation focuses on four main metrics.\looseness=-1
\begin{itemize}[leftmargin=*]
\item \textbf{CI Coverage:} The empirical probability that the nominal $(1-\alpha)$ confidence interval produced by a given method contains the true full-sample parameter ($\theta_i$ or $\Delta_{ij}$).
\item \textbf{\% Reduction in CI Size (or MSE):} The relative decrease in average confidence interval width achieved by a given method compared to the classic independent sampling baseline. When comparing methods that do not produce valid confidence intervals (e.g., the naive imputation estimator), we instead report the \% reduction in mean squared error (MSE).

\item \textbf{Effective Sample Fraction:} The labeling budget the classic independent baseline would require to achieve the exact same MSE as a given method evaluated at fraction $p$. This is plotted as a fraction of the total dataset $N$; values strictly greater than $p$ indicate that the evaluated method yields the statistical power of a larger independent labeled sample.
\item \textbf{Fractional MSE:} The MSE of a given method at fraction $p$, divided by the MSE at  $p=1$ (recall that we evaluate MSE using bootstrap resampling, so the MSE at $p=1$ is $>0$). This quantifies how much of the precision of the full evaluation we can recover with a smaller budget. Lower values are better, with $1$ representing parity with the full size evaluation.\looseness=-1
\end{itemize}

\subsection{Results}

\begin{figure}[!t]
    \centering
    \includegraphics[width=1\linewidth]{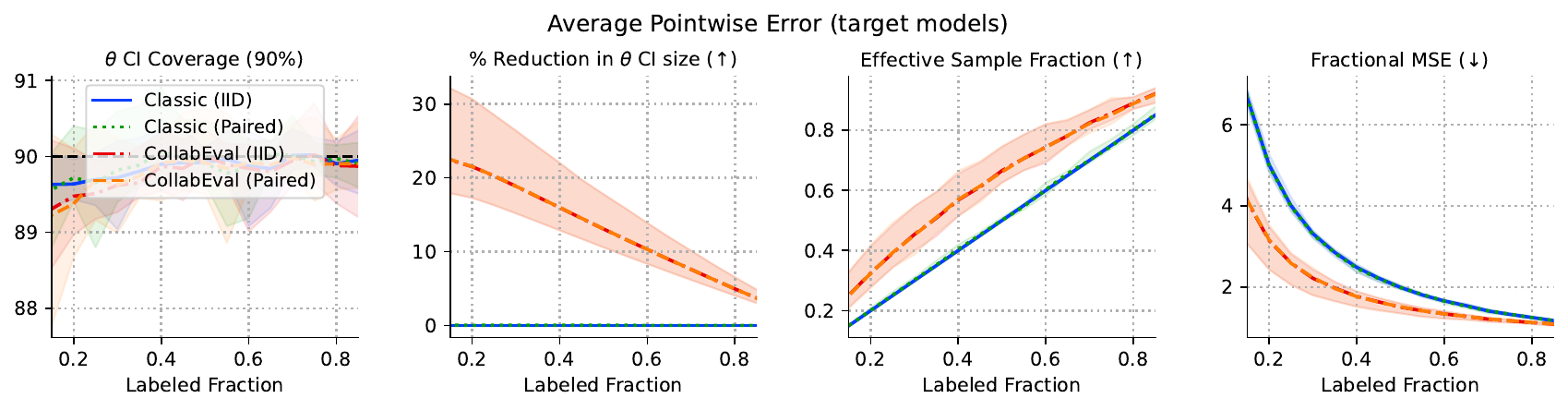}
    \includegraphics[width=1\linewidth]{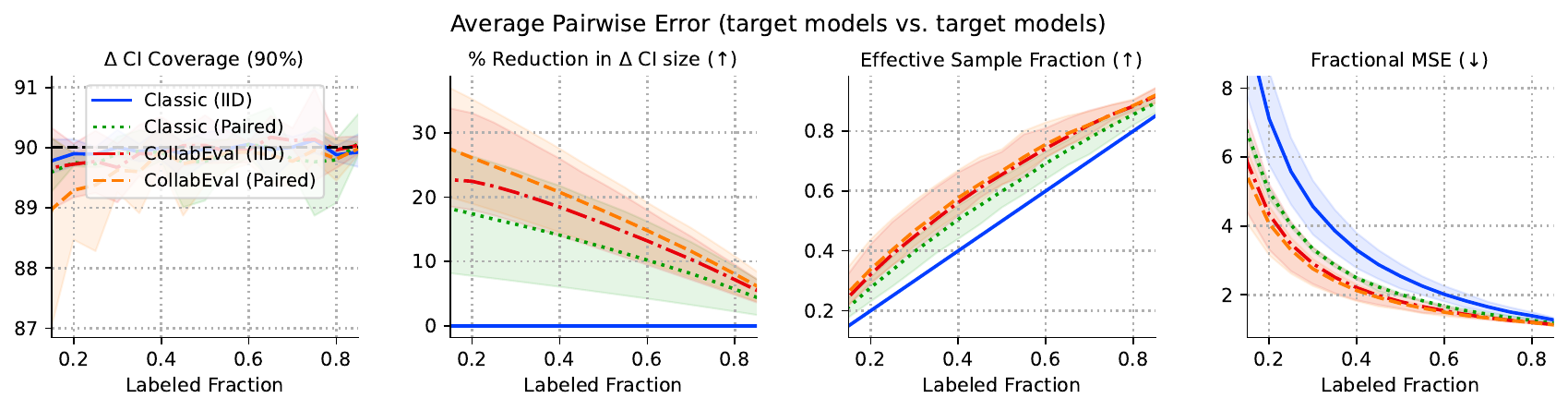}
    \includegraphics[width=1\linewidth]{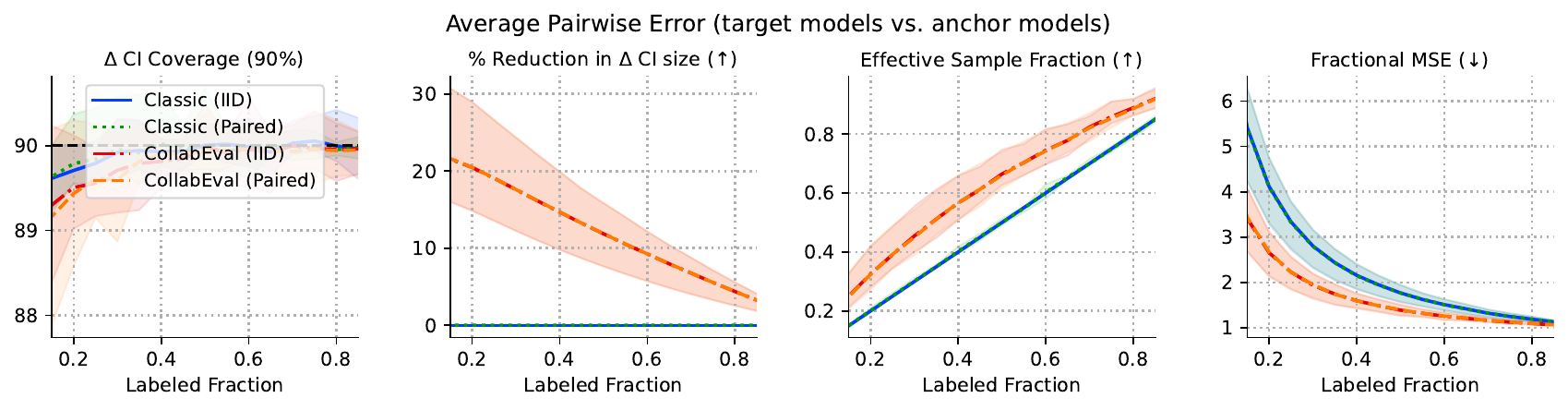}
    \vspace{-12pt}
    \caption{\textbf{Performance metrics averaged across all five datasets.} Rows show pointwise estimation of target models ($\hat{\theta}_i$, top), pairwise differences between target models ($\hat{\Delta}_{ij}$ for $i,j \in \mathcal{T}$, middle), and pairwise differences between target and anchor models ($\hat{\Delta}_{ij}$ for $i \in \mathcal{T}, j \in \mathcal{A}$, bottom). CollabEval consistently improves efficiency over classical baselines (and the additional baselines in Appendix~\ref{sec:ablations})---yielding tighter CIs, lower MSE, and higher effective sample fractions---while maintaining coverage. IID and paired sampling perform equivalently for the top and bottom rows, but for target-vs-target comparisons (middle) paired sampling with CollabEval achieves the highest efficiency, particularly for lower labeled fractions $p$.\looseness=-1}
    \label{fig:averaged_results}
\end{figure}

\textbf{Confidence interval coverage.} CollabEval consistently approaches the nominal coverage rate of 90\% (dashed lines in the first column of Figure~\ref{fig:averaged_results}). This empirically validates our theoretical results that were used to derive the confidence intervals in Section~\ref{sec:method}. While our guarantees are asymptotic in the total number of labels, we note that the coverage empirically quickly approaches the desired coverage rate even in lower sample settings ($p < 0.2$). In contrast, the naive imputation baseline  given in Appendix~\ref{sec:ablations} fails to maintain valid coverage.\looseness=-1

\textbf{Estimation efficiency (CI Size and MSE).} CollabEval demonstrates a significant improvement in statistical efficiency compared to the classical empirical estimators. For individual mean estimation ($\theta_i$), CollabEval yields substantially tighter confidence intervals than the baseline, showing an average reduction in CI width of over 20\%  when the labeled fraction $p$ is small. This effect is even more pronounced on specific datasets; for example, in the AQA evaluation (see Appendix~\ref{sec:per_task}), we observe a reduction in CI width of over 30\% at lower fractions. Similar trends are observed for pairwise differences ($\Delta_{ij}$), where CollabEval reduces the average CI width by over 30\% in the target-vs-target setting, driven by particularly strong performance on the WMT and AQA benchmarks (Appendix~\ref{sec:per_task}). Furthermore, the "Effective Sample Fraction" curves (Figure~\ref{fig:averaged_results}, third column) consistently stay above the $y=x$ parity line, indicating that CollabEval requires fewer ground-truth labels to achieve the same MSE as the baseline. For instance, on average across datasets, when the classic estimator requires 60\% of the labeled data to achieve a certain precision for $\theta$, CollabEval can achieve the same precision with approximately 45\% of the labeled data. Finally, our additional results in Appendix~\ref{sec:ablations} show that CollabEval is significantly better at reducing the MSE than both the naive estimator (which can even sometimes \emph{hurt MSE}) and the PPI (Anchor Mean) baseline---highlighting both the importance of explicit bias correction, as well as the use of the cross-fold matrix completion approach over an untrained alternative. Similarly, Appendix~\ref{sec:ablations} shows that CollabEval consistently yields lower MSE than the Anchor Points baseline, and exhibits greater empirical stability across tasks and labeling fractions.\looseness=-1

\textbf{Effect of the labeled fraction size.} The gains of CollabEval are generally most pronounced  when the labeled fraction $p$ is small. As seen in the "\% Reduction in CI Size" plots, the relative advantage of CollabEval peaks when the labeled data is scarce, but in sufficient quantity to learn from (e.g., $p \approx 0.1-0.2$), and gradually diminishes as $p$ grows. This behavior is expected: as the labeled set becomes larger, two things happen: (1) the variance of the classical estimator naturally decreases, and (2) the relative amount of extra unlabeled data CollabEval can leverage diminishes. However, even at moderate labeled fractions (e.g., $p=0.5$), CollabEval often retains a substantial efficiency advantage---reducing the average CI width by $\sim$10-20\% depending on the specific estimand, and $>20$\% on tasks like AQA and WMT (Appendix~\ref{sec:per_task}). As $p$ gets closer to $1.0$, all methods converge.\looseness=-1

\textbf{Effects of the sampling strategy.} For individual means ($\theta_i$), paired and independent sampling perform virtually identically across our experiments, as seen in the top row of Figure~\ref{fig:averaged_results}. However, for difference estimates between target models only (Figure~\ref{fig:averaged_results}, middle row), the benefits of paired sampling are pronounced. As evaluation scores on the same prompts are often correlated across models, the classical paired estimator is competitive, and outperforms the classical IID estimator (and suggests that paired sampling should generally be the natural approach for these types of comparative evaluation tasks). Building upon this, CollabEval with paired sampling further improves upon the classical paired estimator, yielding the tightest confidence intervals and lowest error overall. Note that these effects disappear when looking at the difference estimates between target models and anchor models (Figure~\ref{fig:averaged_results}, bottom row), since here paired and IID sampling are effectively the same (as every target model observation is automatically paired with an anchor model observation).\looseness=-1

\textbf{Effects of the anchor model set.} CollabEval relies on correlations between anchor and target model evaluations. In Figure~\ref{fig:anchor_ablation} in Appendix~\ref{app:anchor_set}, we ablate the anchor set's size and composition at a fixed 50\% labeled data fraction. We compare three selection strategies for an anchor set of size $k$: highest-scoring (Top-$k$), lowest-scoring (Bottom-$k$), and randomly selected models (Random-$k$). Confidence interval coverage remains stable across all sizes and strategies, consistent with our theoretical guarantees. Efficiency improves as the anchor set grows and becomes more predictive of target performance, though CollabEval maintains positive variance reduction even when the anchor model set is small (for all selection strategies). With respect to differences in anchor selection strategies, random selection (Random-$k$) yields the highest variance reduction, while selecting only the lowest-scoring models (Bottom-$k$) is the least effective strategy (though it still reduces variance relative to classical estimation). These results indicate that while   the \emph{degree} of variance reduction that can be achieved by CollabEval naturally varies depending on the anchor models, its performance is quite robust in general.\looseness=-1

\section{Conclusion}
\label{sec:conclusion}

This work introduces Collaborative Evaluation (CollabEval), a practical and rigorous framework for reducing model evaluation cost. Formalizing the widely reported phenomenon in recent literature that different model outputs and evaluation scores are often highly correlated, CollabEval casts evaluation as a matrix completion problem, and uses historical data to impute scores that are then used as control variates within a cross-prediction-powered inference framework. We theoretically establish that CollabEval yields unbiased estimates with asymptotically valid confidence intervals for both individual model performance and pairwise model comparisons, regardless of the accuracy of the underlying completion (and with only mild assumptions on the algorithm used), or the data distribution. Empirically, we demonstrate across a range of benchmarks that CollabEval significantly outperforms standard estimation techniques. By effectively converting historical evaluations into increased statistical power, CollabEval substantially increases evaluation precision at nearly no extra cost to the evaluator, and serves as a simple but  principled resource for practitioners tasked with the continuous evaluation and monitoring of rapidly evolving generative AI systems.\looseness=-1

\bibliographystyle{plainnat}
\bibliography{bib}
\clearpage

\appendix
\setcounter{theorem}{0}
\numberwithin{theorem}{section}
\addcontentsline{toc}{chapter}{Appendices} 
\etocsetnexttocdepth{2}
\localtableofcontents
\clearpage
\section{Theoretical results}
\label{app:theory}

\subsection{Asymptotic normality}
The purpose of this section is to prove that the estimate $\hat{\theta}$ is asymptotically multivariate normal, which will then justify the form of confidence interval proposed in Section~\ref{sec:interval}. 
To start, we require a stability condition (Assumption \ref{ass:stable}; adapted from \citet{zrnic2025cross} to our setting) on the matrix completion algorithm $\mathcal{M}$. Informally, this requires that as we collect  data, the cross-fold predictions become insensitive to their specific fold  and converge to a fixed population-level imputation $\hat{\bunderline{S}}_{ij}$ that depends only on the observed data for item $j$.\looseness=-1
\begin{assumption}[Imputation stability]\label{ass:stable}
Let $K$ be the number of folds, let $n_i = |\mathcal{J}_i|$ be the number of observed scores for a target model $i$, and let $\hat{\bunderline{S}}_{ij} = h_i(X_j)$ where $X_j = \{S_{i',j} \colon \Omega_{i'j} = 1, i' \in \mathcal{A}\}$ be some deterministic population-limit prediction that depends exclusively on the observed anchor model scores on the $j$-th prompt. Furthermore, let $\hat{S}^{(1)} = \mathcal{M}(S, \Omega^{(1)})$ be the matrix completion evaluated on a random fold mask $\Omega^{(1)}$. We say that $\mathcal{M}$ is stable if as $N \to \infty$ and $n_i / N \to p_i \in (0, 1]$:$$ \sqrt{K \mathrm{Var}\left( \hat{S}^{(1)}_{ij} - \hat{\bunderline{S}}_{ij} \mid \Omega^{(1)} \right)} \overset{L_1}{\longrightarrow} 0,$$
where $\xrightarrow{L_1}$ denotes convergence in mean.
\end{assumption}
Intuitively, this ensures that the row entries of the cross-fold-imputed matrix $\hat{Y}$ from Equation \eqref{eq:cross_fold_matrix} behave like asymptotically i.i.d. random variables as the matrix dimensions grow, as the cross-fold predictions converge to a single, stable predictor that is only a function of the covariates (in this case, the anchor model scores) for prompt $j$. In practice, this form of stability is a typical property of appropriately regularized matrix completion methods.\footnote{It is, for example, possible to argue using standard concentration and perturbation results from high-dimensional statistics (e.g., the Matrix Bernstein inequality and Wedin’s $\sin \Theta$ theorem)  that our cross-fold iterative SVD algorithm concentrates toward a deterministic linear projection onto the principal singular subspace of the population covariance matrix, with mild assumptions. This analysis, however, is outside the primary scope of this paper; we provide this theoretical sketch for intuition and rely on our empirical results to demonstrate the stability and coverage of the algorithm in practice.\looseness=-1}



Under this condition, we establish the following central limit theorem.\looseness=-1
\begin{theorem}[Multivariate CLT for $\hat{\theta}$]
\label{thm:collab_clt}
Suppose Assumption \ref{ass:stable} holds, and that the random variables $(S_{ij}, \hat{\bunderline{S}}_{ij})$ have finite covariance.  Let $n_i = |\mathcal{J}_i|$, $n_j = |\mathcal{J}_j|$, and $n_{ij} = |\mathcal{J}_i \cap \mathcal{J}_j|$, where as $N \to \infty$, $n_i / N \xrightarrow{} p_i \in (0, 1]$, $n_j / N \xrightarrow{} p_j \in (0, 1]$, and $n_{ij} / N \xrightarrow{p} p_{ij} \in [0, 1]$. 
Then as $N \to \infty$, 
\begin{gather}
\sqrt{N}(\hat{\theta} - \theta) \xrightarrow{d} \mathcal{N}(0, \Sigma), \quad \text{where} \nonumber \\
\label{eq:limiting_cov}
\Sigma_{ij} = \frac{p_{ij}}{p_i p_j} \mathrm{Cov}\left(S_i, S_j\right) + \left(\frac{p_{ij}}{p_i p_j} - 1\right) \left[ \mathrm{Cov}\left(\hat{\bunderline{S}}_i, \hat{\bunderline{S}}_j\right) - \mathrm{Cov}\left(\hat{\bunderline{S}}_i, S_j\right) - \mathrm{Cov}\left(S_i, \hat{\bunderline{S}}_j\right) \right].
\end{gather}
\end{theorem}

As $\hat{\theta}$ is asymptotically normal, any linear transformation of $\hat{\theta}$ is also asymptotically normal. Specifically, for any  ${a} \in \mathbb{R}^M$, we have $\sqrt{N} ({a}^\top \hat{{\theta}} - {a}^\top {\theta}) \xrightarrow{d}  \mathcal{N}(0, {a}^\top {\Sigma} {a})$. We can derive the limiting distributions for various metrics by choosing ${a}$ appropriately: for the individual mean $\theta_i$, we set ${a} = \mathbf{e}_i$ (the standard basis vector), and for the pairwise difference $\Delta_{ij}$, we set ${a} = \mathbf{e}_i - \mathbf{e}_j$.

\begin{corollary}[Univariate CLTs for $\hat{\theta}_i$ and $\hat{\Delta}_{ij}$]
\label{cor:linear_functions}
Under the assumptions of Theorem~\ref{thm:collab_clt},
\begin{equation}
\sqrt{N}(\hat{\theta}_i - \theta_i) \xrightarrow{d} \mathcal{N}(0, \Sigma_{ii})
\quad\text{and}\quad
\sqrt{N}(\hat{\Delta}_{ij} - \Delta_{ij}) \xrightarrow{d} \mathcal{N}(0, \Sigma_{ii} + \Sigma_{jj} - 2\Sigma_{ij}).
\end{equation}
\end{corollary}

These results allow us to construct confidence intervals for the true $\theta_i$ and $\Delta_{ij}$, provided any consistent estimate $\hat{\Sigma}$ of the asymptotic covariance matrix $\Sigma$ in Equation~\eqref{eq:limiting_cov}:

\begin{corollary}[Confidence intervals]
\label{cor:collab_ci}
Let $\hat{\Sigma} \xrightarrow{p} \Sigma$, for $\Sigma$ defined in Equation~\eqref{eq:limiting_cov}. Let $z_{1-\alpha/2}$ denote the $(1-\alpha/2)$-quantile of the standard normal distribution for $\alpha \in (0, 1)$, and let
\begin{equation}
\mathcal{C}_\alpha(\hat{\theta}_i) = \left[ \hat{\theta}_i \pm z_{1-\alpha/2}\sqrt{\frac{\hat{\Sigma}_{ii}}{N}} \right]
\quad\text{and}\quad
\mathcal{C}_\alpha(\hat{\Delta}_{ij}) = \left[ \hat{\Delta}_{ij} \pm z_{1-\alpha/2}\sqrt{\frac{\hat{\Sigma}_{ii} + \hat{\Sigma}_{jj} - 2\hat{\Sigma}_{ij}}{N}} \right].
\end{equation}
Then, under the assumptions of Theorem~\ref{thm:collab_clt}, the intervals $\mathcal{C}_\alpha(\hat{\theta}_i)$ and $\mathcal{C}_\alpha(\hat{\Delta}_{ij})$ satisfy:
\begin{equation}
\liminf_{N \to \infty} \mathbb{P}\big(\theta_i \in \mathcal{C}_\alpha(\hat{\theta}_i)\big) \geq 1 - \alpha \quad \text{and} \quad \liminf_{N \to \infty} \mathbb{P}\big(\Delta_{ij} \in \mathcal{C}_\alpha(\hat{\Delta}_{ij})\big) \geq 1 - \alpha.
\end{equation}
\end{corollary}

We are now ready to provide the formal statement of Theorem~\ref{thm:informal_coverage}.

\begin{theorem}[CI Coverage of CollabEval: Formal]
\label{thm:empirical_cov}
Suppose Assumption \ref{ass:stable} holds, the evaluation scores $S_{ij}$ are strictly bounded, and the sampling conditions of Theorem \ref{thm:collab_clt} are satisfied. Let $\hat{\Sigma}$ be the empirical covariance estimator constructed using the cross-fold predictions $\hat{Y}$ and true observed evaluation scores $S$. Then, $\hat{\Sigma} \xrightarrow{p} \Sigma$, and consequently,
\begin{equation}
\liminf_{N \to \infty} \mathbb{P}\big(\theta_i \in \mathcal{C}_\alpha(\hat{\theta}_i)\big) \geq 1 - \alpha \quad \text{and} \quad \liminf_{N \to \infty} \mathbb{P}\big(\Delta_{ij} \in \mathcal{C}_\alpha(\hat{\Delta}_{ij})\big) \geq 1 - \alpha.
\end{equation}
\end{theorem}
\begin{remark}
Note that these intervals guarantee marginal coverage for individual $\theta_i$ and $\Delta_{ij}$. To obtain simultaneous coverage for the full vector of parameters $\theta \in \mathbb{R}^M$, one can construct a confidence ellipsoid based on the $\chi^2$ distribution, or apply a Bonferroni correction by using marginal intervals with level $1 - \alpha/M$. See a similar discussion for multivariate PPI++ in \citet{angelopoulos2023ppi++}, and for deriving intervals for full model rankings from pairwise comparisons in \citet{chatzi2024ppirank}.\looseness=-1
\end{remark}


\subsection{Power tuning}
\label{sec:powertuning}
\citet{angelopoulos2023ppi++} proposed "power tuning" as a way to improve upon the standard PPI estimator by adapting to the quality of the predictions $\hat{Y}$ through the use of a learned scalar $\lambda\in\mathbb{R}$ that can reduce variance when setting $\hat{Y}_i^{\prime}=\lambda_i\hat{Y}_i$. We extend this to our setting to learn a finite set of weights $\hat{\lambda}\in\mathbb{R}^{M}$ for the predictions $\hat{Y}_{i}$ of each row $i$. By Slusky's theorem, this  weighting preserves asymptotic normality as long as $\hat{\lambda} \xrightarrow[]{p} \lambda$, and is picked to minimize variance. The power-tuned estimate for $\theta_{i}$ is:\looseness=-1
$$\hat{\theta}_{i}^{\lambda_i}=\frac{1}{|\mathcal{J}_i|} \sum_{j \in \mathcal{J}_i} S_{ij} - \lambda_i\left(\frac{1}{|\mathcal{J}_i|} \sum_{j \in \mathcal{J}_i} \hat{Y}_{ij} - \frac{1}{N}\sum_{j=1}^N \hat{Y}_{ij}\right),$$ while the power-tuned estimate for the difference is $\hat{\Delta}_{ij}^{\lambda_{ij}}=\hat{\theta}_{i}^{\lambda_i}-\hat{\theta}_{j}^{\lambda_j}$. Optimal values for $\lambda$  differ depending on if the goal is to estimate $\theta_{i}$ and $\theta_{j}$ individually, or their difference $\Delta_{ij}$.\looseness=-1

\begin{proposition}[Optimal weights for $\hat{\theta}_i^\lambda$ and $\hat{\Delta}_{ij}^\lambda$]
\label{cor:optimal_lambda}
Assume ${\mathrm{Var}(\hat{\bunderline{S}}_i)} > 0$. Then $\lambda^*_i = \frac{\mathrm{Cov}(S_i, \hat{\bunderline{S}}_i)}{\mathrm{Var}(\hat{\bunderline{S}}_i)}$ minimizes the asymptotic variance of  $\hat{\theta}_i^{\lambda_i}$.
Further assume ${\mathrm{Var}(\hat{\bunderline{S}}_j)} > 0$. Let $\gamma_{ij} = \frac{p_{ij}}{p_ip_j} - 1$. Then
\begin{equation}
\label{eq:lambda_delta}
\begin{bmatrix}
\lambda^*_i \\
\lambda^*_j
\end{bmatrix}
=
\underbrace{\begin{bmatrix}
\gamma_{ii} \mathrm{Var}(\hat{\bunderline{S}}_i) & -\gamma_{ij} \mathrm{Cov}(\hat{\bunderline{S}}_i, \hat{\bunderline{S}}_j) \\
-\gamma_{ij} \mathrm{Cov}(\hat{\bunderline{S}}_i, \hat{\bunderline{S}}_j) & \gamma_{jj} \mathrm{Var}(\hat{\bunderline{S}}_j)
\end{bmatrix}^{-1}}_{{Q}}
\underbrace{\begin{bmatrix}
\gamma_{ii} \mathrm{Cov}(S_i, \hat{\bunderline{S}}_i) - \gamma_{ij} \mathrm{Cov}(\hat{\bunderline{S}}_i, S_j) \\
\gamma_{jj} \mathrm{Cov}(S_j, \hat{\bunderline{S}}_j) - \gamma_{ij} \mathrm{Cov}(S_i, \hat{\bunderline{S}}_j)
\end{bmatrix}}_{{u}}
\end{equation}
minimizes the asymptotic variance of $\hat{\Delta}_{ij}^{\lambda_{ij}}$.
\end{proposition}
\begin{remark}
Note that the joint optimal weights $(\lambda^*_i, \lambda^*_j)^\top$ defined in Equation (\ref{eq:lambda_delta}) are equivalent to the individual optimal weights $\lambda^*_i = \mathrm{Cov}(S_i, \hat{\bunderline{S}}_i) / \mathrm{Var}(\hat{\bunderline{S}}_i)$ and $\lambda^*_j = \mathrm{Cov}(S_j, \hat{\bunderline{S}}_j) / \mathrm{Var}(\hat{\bunderline{S}}_j)$ when the cross-model terms are zero. This occurs under either of the following conditions:
\begin{itemize}[leftmargin=*, itemsep=1pt]
    \item The observation sets are independent such that $\gamma_{ij} = 0$. As $N \to \infty$, this implies the joint observation fraction $p_{ij}$ equals the product of individual fractions $p_i p_j$.
    \item The predicted scores of one model provide no statistical information regarding the true scores or predictions of the other model, specifically $\mathrm{Cov}(\hat{\bunderline{S}}_i, \hat{\bunderline{S}}_j) = 0$ and $\mathrm{Cov}(\hat{\bunderline{S}}_i, S_j) = 0$.
\end{itemize}
In these cases, the matrix $Q$ becomes diagonal and the vector $u$ simplifies, allowing the joint system to decouple into separate univariate optimizations. However, in practical settings where models are highly correlated or paired sampling is used ($\mathcal{J}_i = \mathcal{J}_j$), the joint weights typically diverge from individual ones to better minimize the variance of the difference estimate.
\end{remark}

Power-tuning ensures that the asymptotic variance of either of our estimates $\hat{\theta}_i^{\lambda_i^*}$ and $\hat{\Delta}_{ij}^{\lambda_{ij}^*}$ is at most that of the classical empirical estimator. To see this, let $\hat{\theta}_i^\mathrm{classic} = \frac{1}{|\mathcal{J}_i|} \sum_{j \in \mathcal{J}_i} S_{ij}$ be the  empirical mean estimate, and let $\hat{\Delta}_{ij}^\mathrm{classic} = \hat{\theta}_i^\mathrm{classic} - \hat{\theta}_j^\mathrm{classic}$ be the empirical mean difference estimate. The following result then gives the asymptotic variance of the power-tuned estimate in terms of the variance of the corresponding classical estimate, less a non-negative term:\looseness=-1

\begin{proposition}
\label{prop:power_tuning_distribution}
Let $\hat{\lambda} \overset{p}{\rightarrow} \lambda^*$, where $\lambda^*$ is defined accordingly for $\hat{\theta}_i^{\lambda_i^*}$ and $\hat{\Delta}_{ij}^{\lambda_{ij}^*}$. Then
\begin{align}
\sqrt{N}(\hat{\theta}_i^{\hat{\lambda}_i} -\theta_i) &\overset{d}{\rightarrow} \mathcal{N}\left(0, \mathrm{Var}(\hat{\theta}_i^\mathrm{classic}) - \left(\frac{1}{p_i} - 1\right)\frac{\mathrm{Cov}(S_i, \hat{\bunderline{S}}_i)^2}{\mathrm{Var}(\hat{\bunderline{S}}_i)}\right) \quad \text{and} \\
\sqrt{N}(\hat{\Delta}_{ij}^{\hat{\lambda}_{ij}} - \Delta_{ij}) &\overset{d}{\rightarrow} \mathcal{N}(0, \mathrm{Var}(\hat{\Delta}_{ij}^\mathrm{classic}) - u^\top Q^{-1}u),
\end{align}
for $u$ and $Q$ defined in Equation~\eqref{eq:lambda_delta},  assuming ${\mathrm{Var}(\hat{\bunderline{S}}_i)} > 0$ and $\mathrm{Var}(\hat{\bunderline{S}}_j) > 0$.
\end{proposition}
\begin{remark} Note that $Q^{-1}$ is PSD, which implies $u^\top Q^{-1} u \geq 0$.
\end{remark}
We use the empirical covariance matrix of $(S, \hat{S})$ to estimate $\lambda^*$.

\section{Dataset details}
\label{sec:datasets}
\textbf{AlpacaEval 2.0.} AlpacaEval~\cite{alpaca_eval} evaluates the instruction-following capabilities of AI models. Models generate responses to user instructions, which an AI autorater (GPT-4 Turbo) then scores against a baseline (also based on GPT-4 Turbo) to determine a win-rate. We select the latest 17 submissions to the AlpacaEval Github repository as our target models, and use the  remaining 204 submissions that were submitted before them as anchor models.

\textbf{MMLU.} The Massive Multitask Language Understanding (MMLU) benchmark~\cite{hendrycks2020measuring} scores over 14k multiple-choice questions for binary accuracy. We evaluate six target LLMs from the Open LLM Leaderboard --- Amber-6.7B~\citep{liu2023llm360}, OLMo1/2-7B~\citep{groeneveld2024olmo, olmo20242}, Pythia-2.8B/6.9B~\citep{biderman2023pythia}, K2-65B~\citep{liu2025llm360} --- alongside 102 anchor LLMs, following \citet{hofmann2025fluid}.\looseness=-1

\textbf{Attributed Question Answering (AQA).} AQA~\cite{bohnet2023attributed} assesses whether a generated answer is both correct and supported by provided evidence, yielding an NLI-based attribution score. We evaluate 21 systems from \citet{bohnet2023attributed} using a leave-two-out protocol: two models serve as targets while the remaining 19 act as anchors. This enables both target-vs-target and target-vs-anchor model comparisons.

\textbf{WMT24$++$ Machine Translation Benchmark.} WMT24$++$ evaluates translations of 998 English paragraphs into 55 languages~\cite{deutsch-etal-2025-wmt24}. Full evaluation of a new model requires both consuming and generating nearly two million tokens ($55$ languages $\times$ $998$ paragraphs $\times \sim 35$ tokens per paragraph). We use the initial release of 15 translation systems provided  by \citet{deutsch-etal-2025-wmt24}, and flatten all language-model pairs into an $825 \times 998$ evaluation matrix that uses MetricX~\cite{juraska-etal-2024-metricx} as the evaluation score. Applying a leave-two-out protocol similar to the AQA setting  within each language yields $2 \times 55$ target rows. However, we restrict all comparisons and paired sampling to be strictly within-language (e.g., we compare model $A$ vs. model $B$ on en-de translations separately from en-fr).\looseness=-1
%

\textbf{SWE-bench.} SWE-bench~\cite{jimenez2024swebench} evaluates the ability of language models to resolve real-world software engineering issues from GitHub. Given a codebase and an issue description, models are tasked with generating a patch that fixes the bug or implements the requested feature. We evaluate the success rate of 12 systems for which item-level scores are publicly available from the SWE-bench leaderboard using the mini-SWE-agent V2 harness, and use the same leave-one-out and leave-two-out protocol as our AQA experiment.

\section{Additional experimental results}
\label{app:additional_results}

\subsection{Effective rank analysis}

Figure~\ref{fig:singular_values} plots the top 32 normalized singular values $(\sigma_i / \sigma_1)$  across all five datasets. The rapid decay of the singular values and the rapid increase of the cumulative explained variance support the hypothesis that the score matrices have low effective rank (which CollabEval can exploit).

\label{sec:singular_values}
\begin{figure}[!h]
    \centering
    \includegraphics[width=1\linewidth]{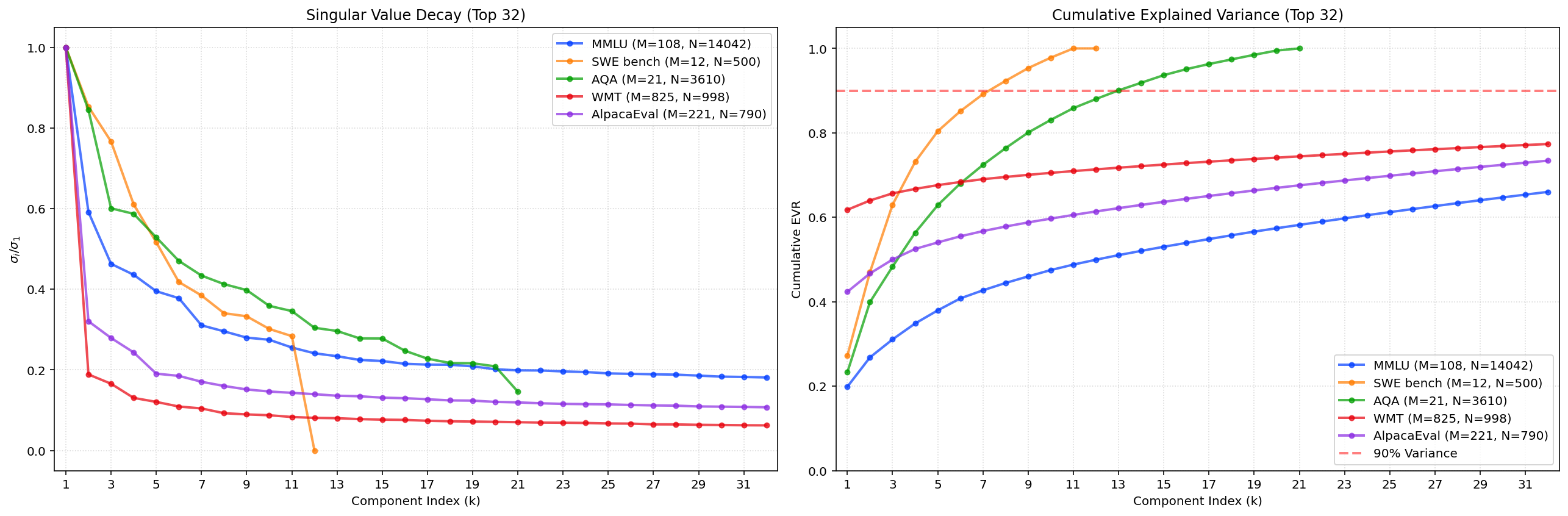}
    \vspace{-12pt}
    \caption{Singular value decay of the  evaluation data matrices for the tasks we consider.}
    \label{fig:singular_values}
   \vspace{-10pt}
\end{figure}

\subsection{Additional baselines}
\label{sec:ablations}

Figure~\ref{fig:additional_baselines} compares the mean squared error of CollabEval against naive imputation, Anchor Points, and standard prediction-powered inference. The Naive Estimate uses the same matrix completion algorithm (IterativeSVD) as CollabEval, but does not correct for bias using the cross-fold procedure. For these results, CollabEval, the naive estimate, and PPI all use paired sampling (Anchor Points also uses paired samples by design). CollabEval consistently yields the lowest error across varying labeled data fractions. Furthermore, while Anchor Points can perform quite well at higher labeled fractions, it also exhibits high variance across datasets---whereas CollabEval is empirically quite robust.

\begin{figure}[!h]
    \centering
    \includegraphics[width=1\linewidth]{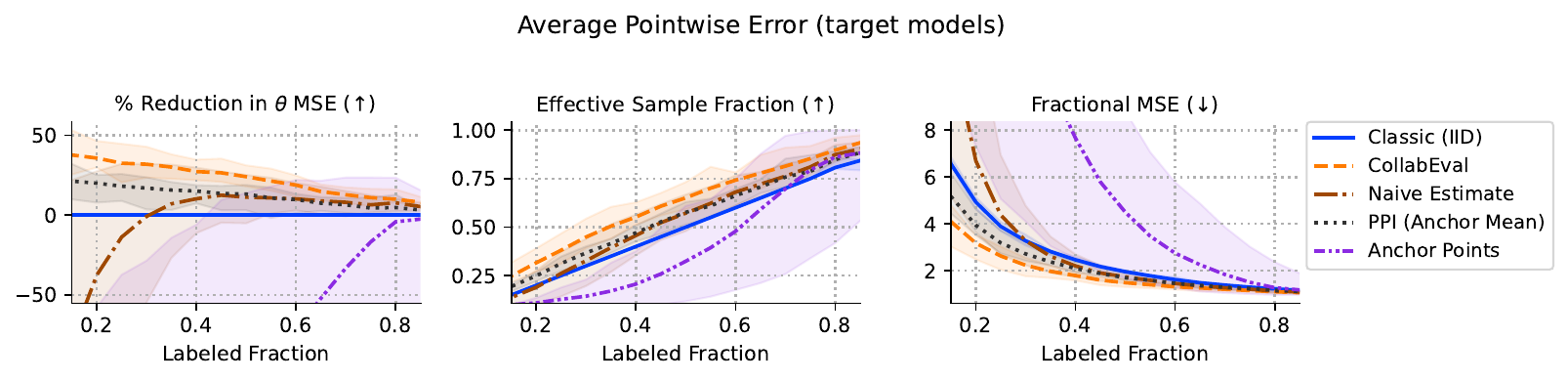}
    \includegraphics[width=1\linewidth]{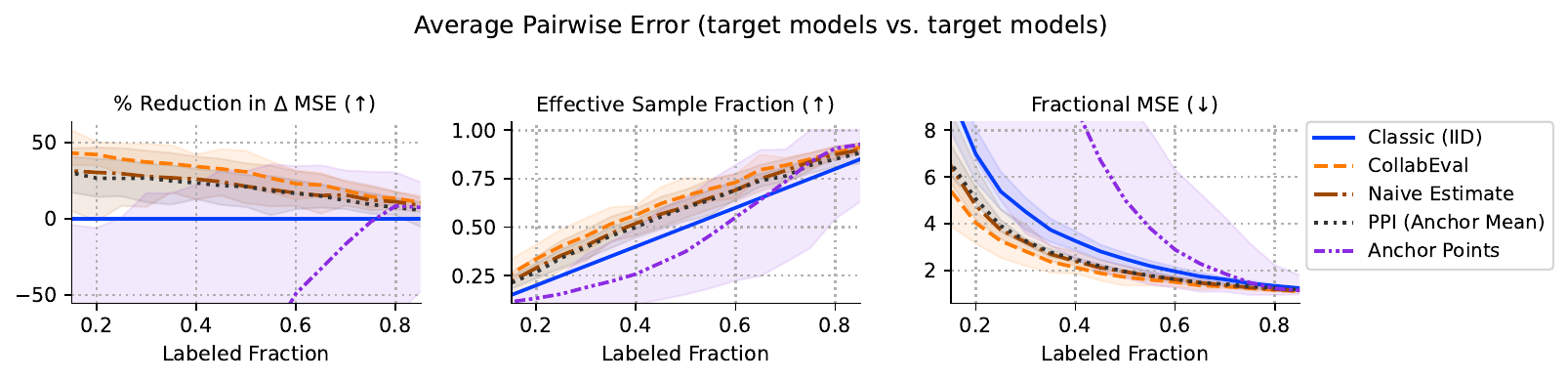}
    \includegraphics[width=1\linewidth]{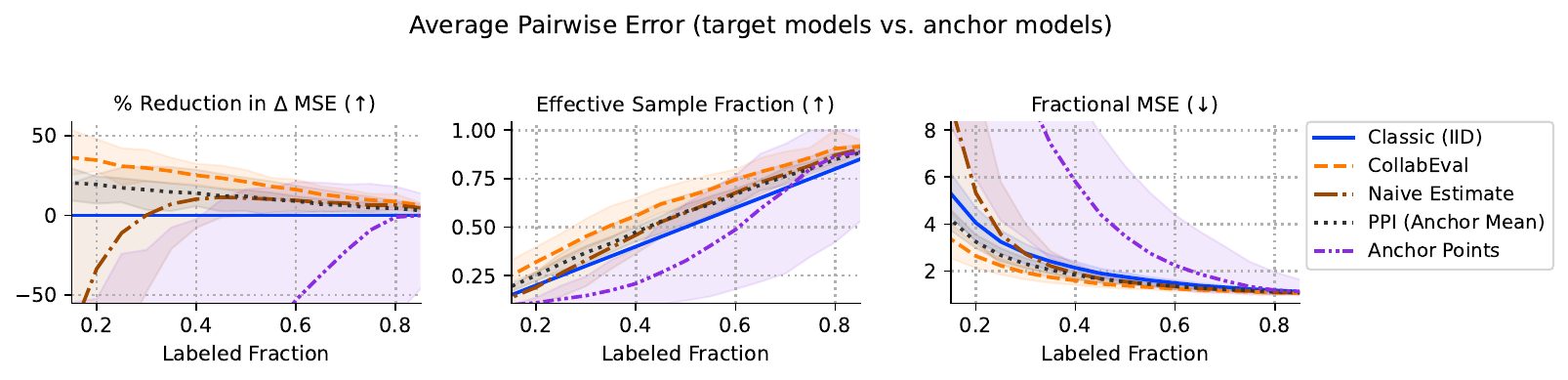}
    \vspace{-12pt}
    \caption{Results when comparing CollabEval to naive imputation, Anchor Points, and standard PPI with per-prompt anchor means. Shaded bands show the worst and best performance across datasets.\looseness=-1}
   \vspace{-10pt}
   \label{fig:additional_baselines}
\end{figure}

\subsection{Per-task results}
 \label{sec:per_task}
 While the average performance in Figure~\ref{fig:averaged_results} demonstrates the overall efficacy of CollabEval, examining the results on a per-dataset basis reveals specific strengths and other interesting dynamics depending on the evaluation context. For instance, on AQA (Figure~\ref{fig:aqa}) we see  particularly strong performance for individual mean estimation, with over a 30\% reduction in confidence interval (CI) width when $p$ is small, and up to a 20\% reduction even when half the data was labeled ($p=0.5$). Similarly, the WMT24++ Machine Translation evaluation (Figure~\ref{fig:wmt}) shows large efficiency gains for pairwise difference estimations, and yields CI width reduction of up to over 30\% in target-vs-target settings at low sampling rates, and also maintains up to a 20\% reduction at moderate sampling rates.
 
 Across these benchmarks, the empirical results further highlight how the choice of sampling strategy interacts with the estimand; while independent (IID) and paired sampling performed virtually identically for estimating individual means, paired sampling is advantageous for estimating differences between target models. In these comparative target-vs-target settings, CollabEval combined with paired sampling consistently yielded the tightest confidence intervals and lowest overall error, noticeably improving upon the already competitive classical paired estimator.

\begin{figure}[!h]
    \centering
    \includegraphics[width=1\linewidth]{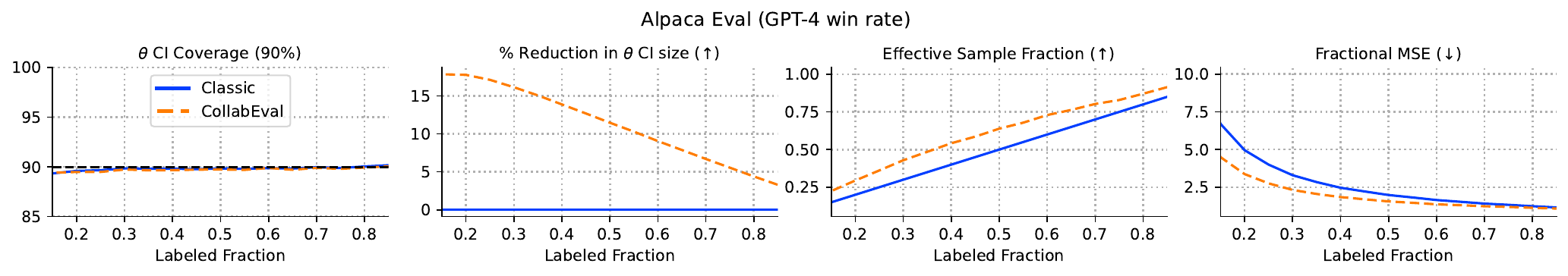}
    \includegraphics[width=1\linewidth]{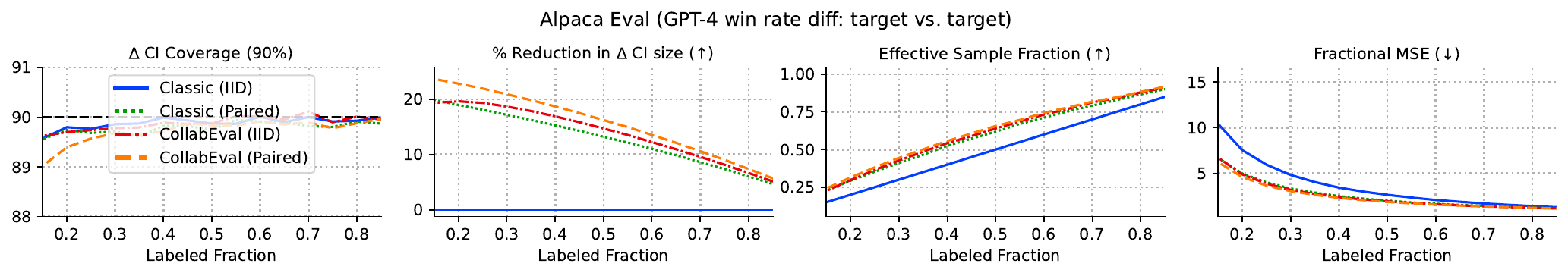}
    \includegraphics[width=1\linewidth]{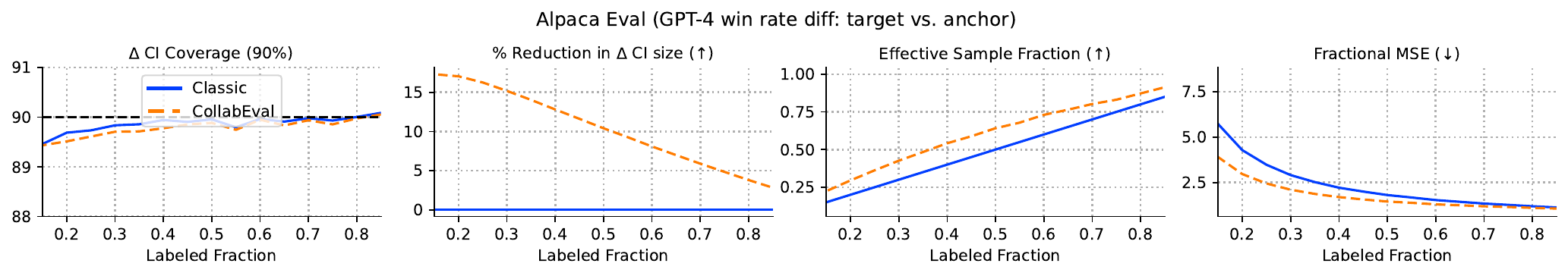}
    \vspace{-15pt}
    \caption{Results for AlpacaEval.}
    \label{fig:alpaca_eval}
\end{figure}

\begin{figure}[!h]
    \centering
    \includegraphics[width=1\linewidth]{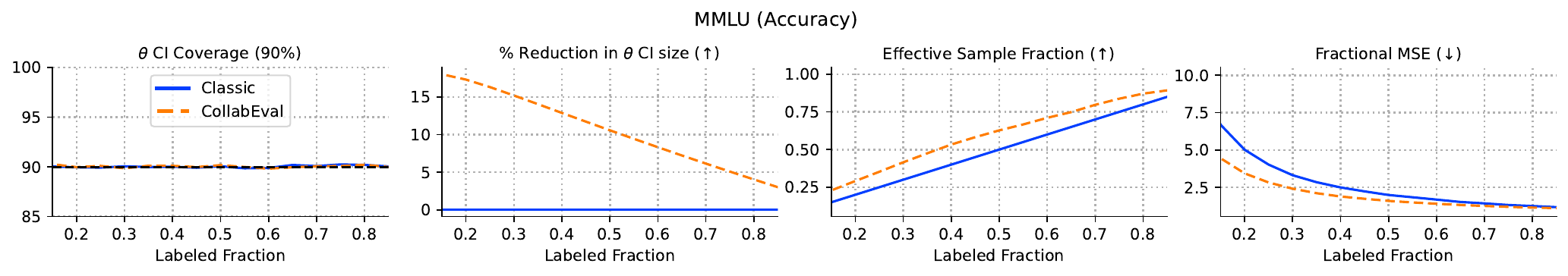}
    \includegraphics[width=1\linewidth]{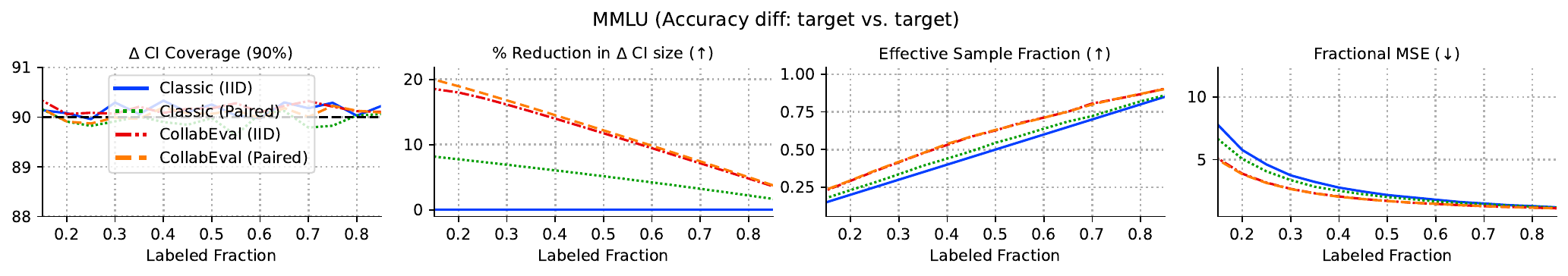}
    \includegraphics[width=1\linewidth]{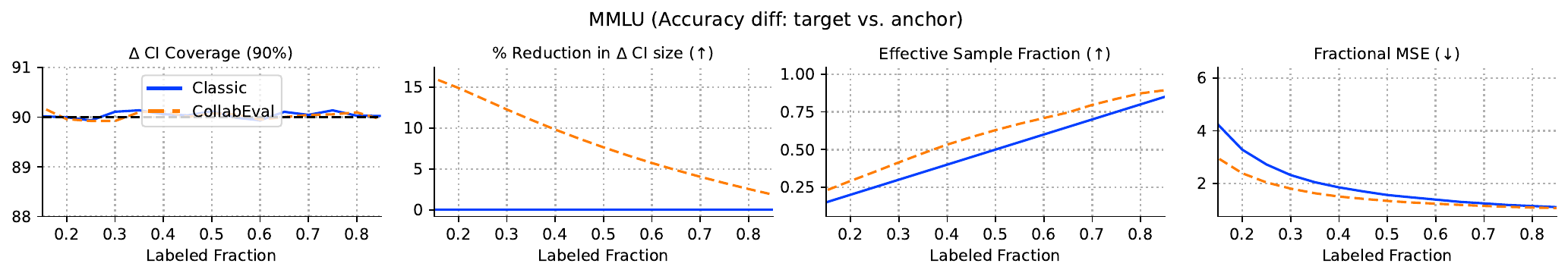}
    \vspace{-15pt}
    \caption{Results for MMLU.}
    \label{fig:mmlu}
\end{figure}

\begin{figure}[!h]
    \centering
    \includegraphics[width=1\linewidth]{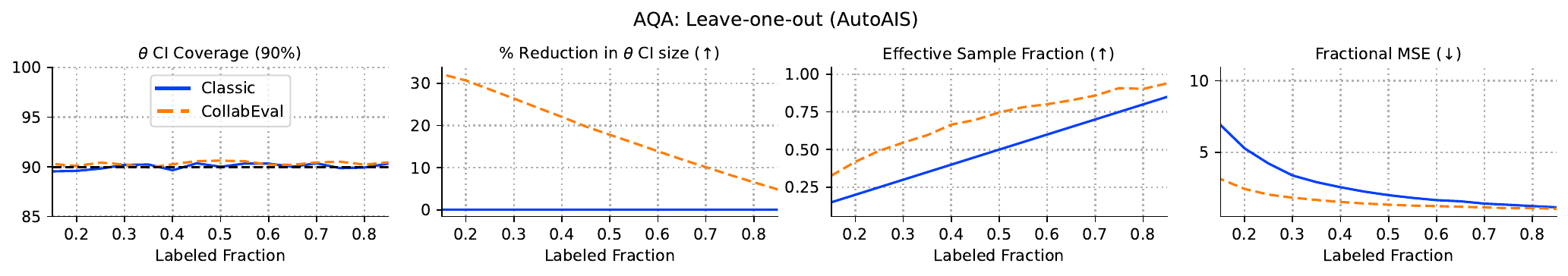}
    \includegraphics[width=1\linewidth]{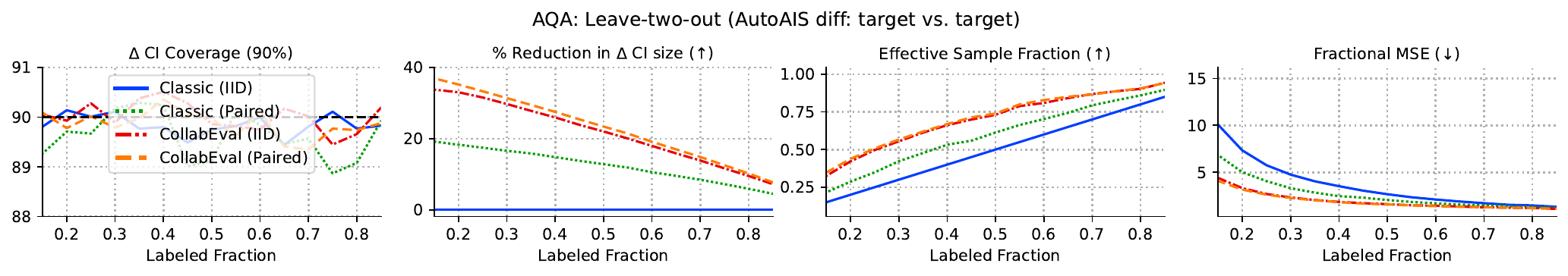}
    \includegraphics[width=1\linewidth]{figures/AQA_Leave-two-out_delta_tgt_tgt.pdf}
    \vspace{-15pt}
    \caption{Results for AQA.}
    \label{fig:aqa}
\end{figure}

\begin{figure}[!h]
    \centering
    \includegraphics[width=1\linewidth]{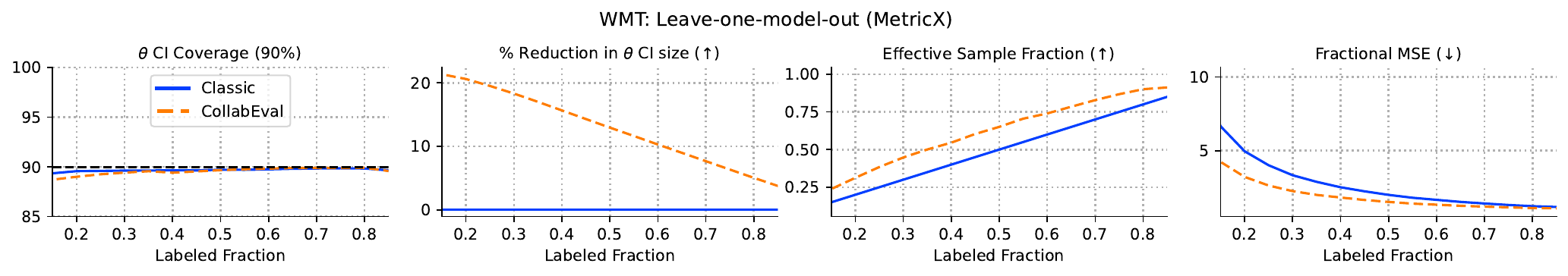}
    \includegraphics[width=1\linewidth]{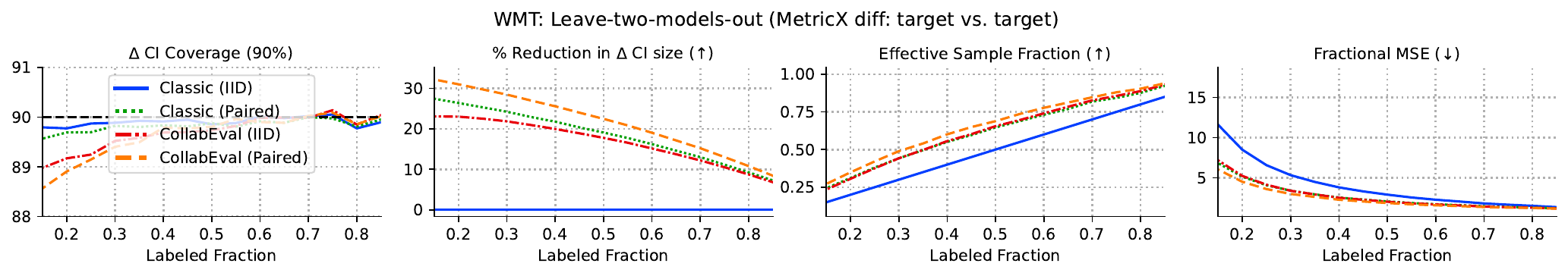}
    \includegraphics[width=1\linewidth]{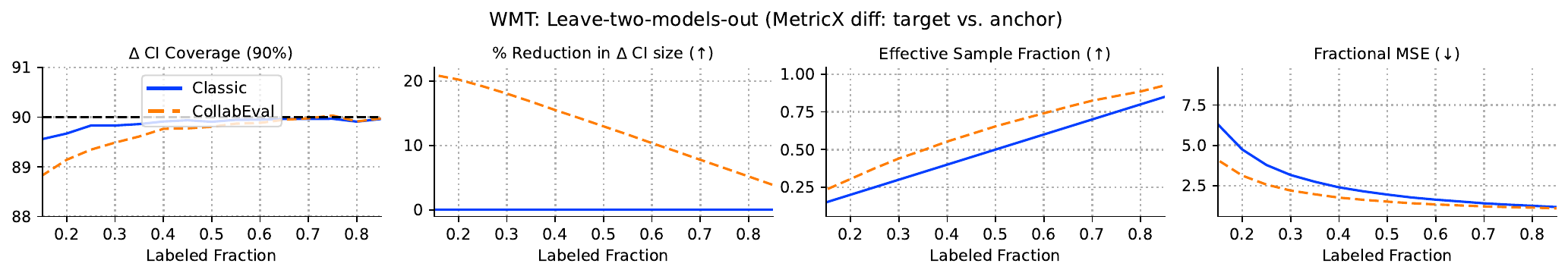}
    \vspace{-15pt}
    \caption{Results for WMT24++.}
    \label{fig:wmt}
\end{figure}

\begin{figure}[!t]
    \centering
    \includegraphics[width=1\linewidth]{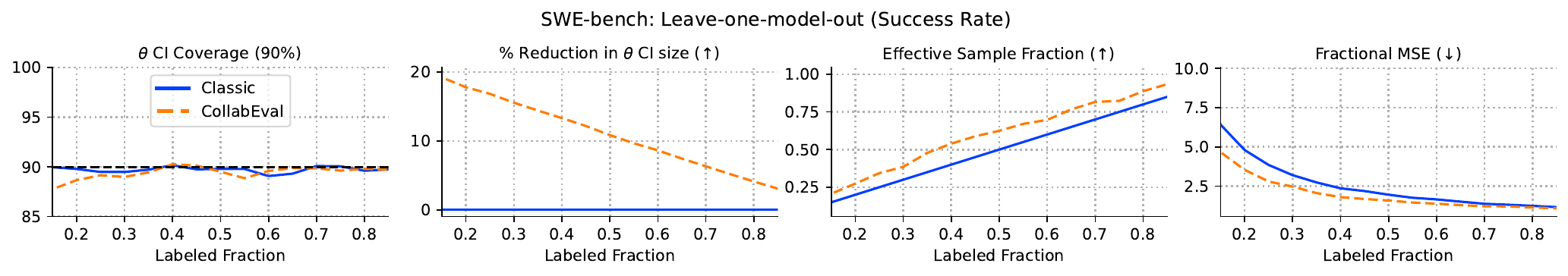}
    \includegraphics[width=1\linewidth]{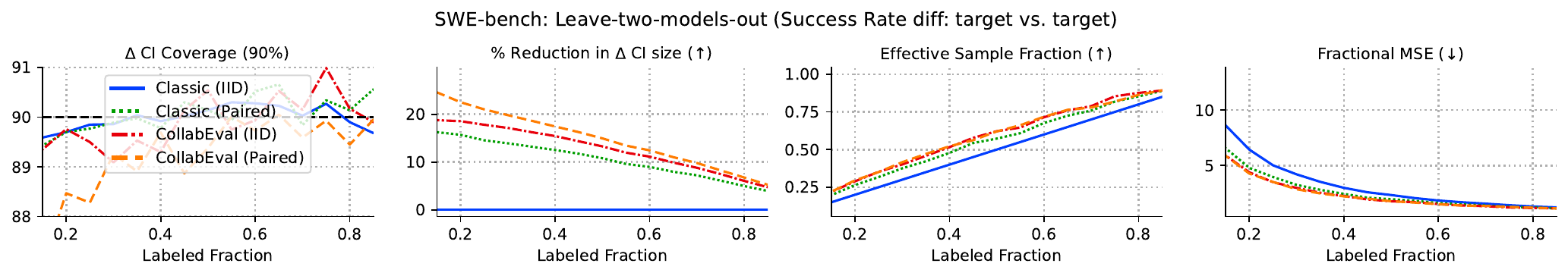}
    \includegraphics[width=1\linewidth]{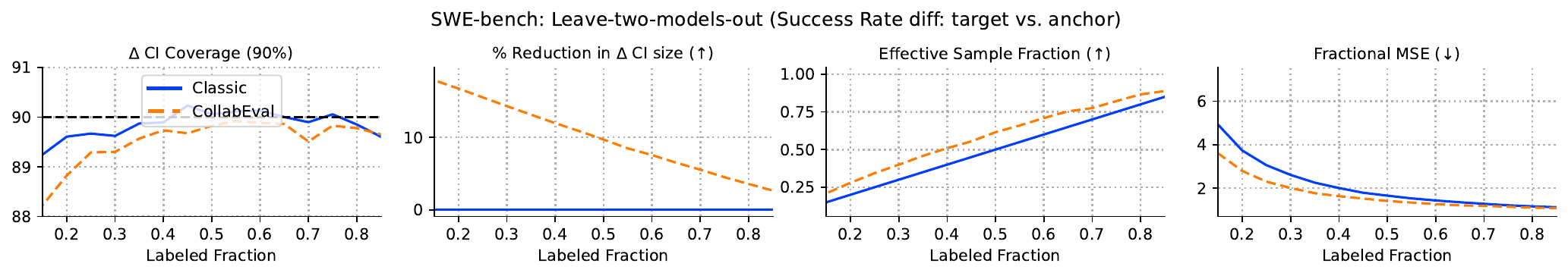}
    \vspace{-15pt}
    \caption{Results for SWE-bench.}
    \label{fig:swebench}
\end{figure}
\clearpage

\subsection{Anchor set ablation}
\label{app:anchor_set}

\begin{figure}[!h]
    \centering
    \includegraphics[width=1\linewidth]{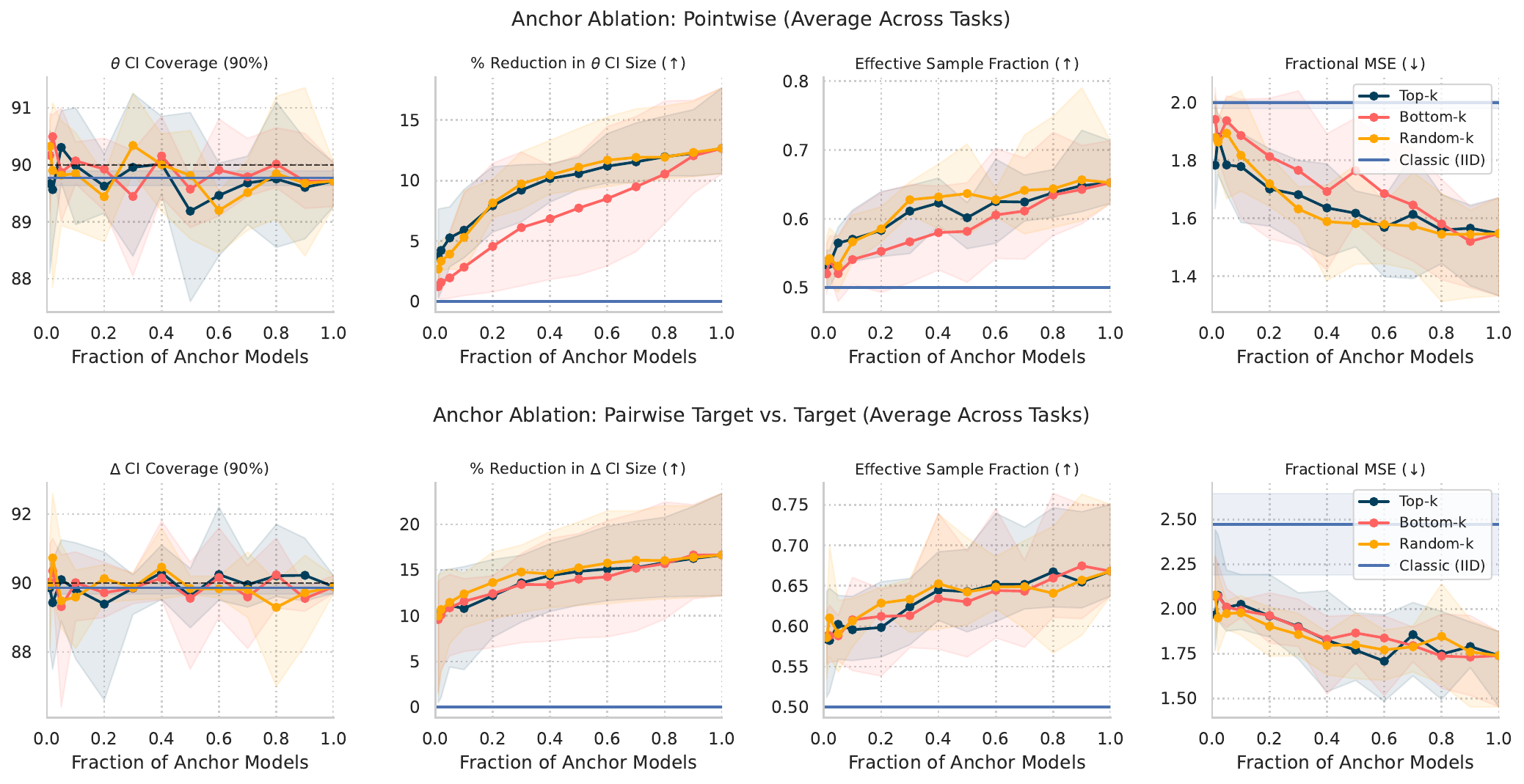}
    \vspace{-12pt}
    \caption{Performance metrics averaged across all five datasets when varying the number of anchor models in the anchor set $\mathcal{A}$, for a fixed labeled data fraction of $50$\%. Shaded bands show the worst and best performance across datasets. We test three selection strategies: (1) Top-k, in which we keep only the highest scoring anchor models in $\mathcal{A}$; Bottom-k, in which we keep only the lowest scoring anchor models in $\mathcal{A}$; and Random-k in which we sample k random models from $\mathcal{A}$. We choose k as a percentage of the total number of available anchors starting at 2\% and always include at least 1 anchor, i.e., \texttt{k = max(1, int(pct * total\_anchors)}). Note that there is some variability in the total number of anchors across tasks. As supported by our theory, CI coverage is largely unaffected by the anchor set, while the estimation efficiency improves as the number of anchor models increases (and becomes more predictive of the target models). As also expected, unbiased selection of anchors (i.e., Random-k) generally gives the best results, by a small margin. We do not report pairwise metrics for target vs. anchor models, as the anchor set changes at each iteration. \looseness=-1}
    \label{fig:anchor_ablation}
\end{figure}
\section{Matrix completion algorithm details and ablations}
\label{sec:implementation}
This appendix presents algorithms for iterative matrix completion (\Cref{alg:isvd}), cross-fold matrix completion (\Cref{alg:xfold_complete}), CollabEval estimation for individual estimands (\Cref{alg:collab_eval_theta}), and estimation for pairwise comparisons (\Cref{alg:collab_eval_delta}). The default sequence of ranks that we use for IterativeSVD is $\mathcal{R} = (1,2,4,8,16,16,16,16)$, and the default number of folds is $K = 10$. These were not extensively tuned; we picked a setting that had a good  tradeoff between speed and predictive ability. IterativeSVD is based on HardImpute~\cite{mazumder2010spectral} which is intended to iterate until convergence with progressively higher ranks.  Section~\ref{sec:isvd_ablation} provides an ablation across datasets at a fixed labeled fraction of 50\% that suggests that performance improves as the max rank in $\mathcal{R}$ increases. Section~\ref{sec:mc_ablation} compares IterativeSVD to several alternative matrix completion algorithms: a neural network baseline, nuclear norm minimization, and an item response theory (IRT) model.

\definecolor{darkgreen}{rgb}{0.31, 0.47, 0.26}
\definecolor{darkorange}{RGB}{153, 51, 0}   
\definecolor{darkblue}{RGB}{0, 0, 180}

\begin{figure}[!h]
{\footnotesize
\begin{algorithm}[H]
\caption{Iterative Matrix Completion via Truncated SVD (\texttt{IterativeSVD})}
\label{alg:isvd}
{
\textbf{Definitions:} $S \in \mathbb{R}^{M \times N}$ is the incomplete label matrix ($M$ models, $N$ items). $\Omega \in \{0, 1\}^{M \times N}$ is the binary mask of observed entries. $\mathcal{R}$ is a sequence of ranks. $\mathtt{SVD}_r(A) = \sum_{k=1}^r \sigma_k u_k v_k^\top$, where $\sigma_k, u_k, v_k$ are the $k$-th singular value and singular vectors of $A$ respectively.
}
\begin{algorithmic}[1]
\State{\textcolor{darkgreen}{\# Initialize missing entries of $S$ with row means.}}
\State{$\hat{Y} = \mathtt{RowMean}(S, \Omega)$}
\For{$r \in \mathcal{R}$}
    \State{\textcolor{darkgreen}{\# Compute rank-$r$ truncated SVD reconstruction and update unobserved entries.}}
    \State{$\hat{Y}_{new} = \mathtt{SVD}_r(\hat{Y})$}
    \State{$\hat{Y}_{ij} = \hat{Y}_{new, ij} \quad \text{for } (i, j) \notin \Omega$}
\EndFor
\Return{$\hat{Y}$}
\end{algorithmic}
\end{algorithm}
}
\end{figure}

\begin{figure}[!h]
{\footnotesize
\begin{algorithm}[H]
\caption{Cross-Fold Matrix Completion (\texttt{XFoldComplete})}
\label{alg:xfold_complete}
{
\textbf{Definitions:} $S \in \mathbb{R}^{M \times N}$ is the incomplete label matrix ($M$ models, $N$ items). $\Omega \in \{0, 1\}^{M \times N}$ is the binary mask of observed entries, where $\mathcal{J}_i = \{j \colon \Omega_{ij} = 1\}$. $\mathcal{T}$ is the set of target models. $K$ is the number of folds. $\mathcal{R}$ is a sequence of ranks. The default is (1, 2, 4, 8, 16, 16, 16, 16).
}
\begin{algorithmic}[1]
\State{Identify active prompts $\mathcal{U} = \bigcup_{i \in \mathcal{T}} \mathcal{J}_i$ and partition them into $K$ disjoint folds $\mathcal{U}_1, \ldots, \mathcal{U}_K$.}
\For{$k = 1, \dots, K$}
    \State{\textcolor{darkgreen}{\# Define training mask and complete fold using iterative SVD.}}
    \State{Define $\Omega_{train}^{(k)}$ such that $\Omega_{train, ij}^{(k)} = 0$ if $i \in \mathcal{T}$ and $j \in \mathcal{U}_k$, and $\Omega_{ij}$ otherwise.}
    \State{$\hat{Y}^{(k)} = \mathtt{IterativeSVD}(S, \Omega_{train}^{(k)}, \mathcal{R})$}
\EndFor
\State{\textcolor{darkgreen}{\# Aggregate predictions for unobserved and anchor entries.}}
\State{$\hat{Y}_{ij} = \hat{Y}^{(k)}_{ij} \text{ if } (i, j) \in \Omega \setminus \Omega_{train}^{(k)}, \text{ and } \frac{1}{K}\sum_{k=1}^K \hat{Y}^{(k)}_{ij} \text{ otherwise.}$} 
\State{}
\Return{$\hat{Y}$}
\end{algorithmic}
\end{algorithm}
}
\end{figure}

\begin{figure}[!h]
{\footnotesize
\centering
\begin{algorithm}[H]
\caption{CollabEval for Single Model Performance (\texttt{EstimateTheta})}
\label{alg:collab_eval_theta}
{
\textbf{Definitions:} $\hat{Y}$ is the completed matrix from Algorithm~\ref{alg:xfold_complete}. $S$ is the incomplete matrix of true evaluation scores. $\Omega$ is the observation mask. $\alpha$ is the desired coverage of the CI for $\theta_i$.
}
\\
{
\textbf{Input:} $i$, the index of the model.
}
\begin{algorithmic}[1]
    \State{\textcolor{darkgreen}{\# Get raw mean estimate + control variate.}}
    \State{$\hat{\mu}_i = \frac{1}{n_i} \sum_{j \in \mathcal{J}_i} S_{ij} \quad \text{where } \mathcal{J}_i = \{j \colon \Omega_{ij} = 1\}\text{ and }n_i = |\mathcal{J}_i|$}
    \State{$\widehat{CV}_i = \frac{1}{n_i} \sum_{j \in \mathcal{J}_i} \hat{Y}_{ij} - \frac{1}{N} \sum_{j=1}^N \hat{Y}_{ij}$\newline}
    
    \State{\textcolor{darkgreen}{\# Power-tune single model scaling, using empirical covariance estimates.}}
    \State{$\hat{\lambda}_i = \frac{\widehat{\mathrm{Cov}}_{j \in \mathcal{J}_i}(S_{ij}, \hat{Y}_{ij})}{\widehat{\mathrm{Var}}_{j \in [N]}(\hat{Y}_{ij})}$\newline}

    \State{\textcolor{darkgreen}{\# Combine raw mean estimate with weighted control variate.}}
    \State{$\hat{\theta}^{\hat{\lambda}_i}_i = \hat{\mu}_i - \hat{\lambda}_i \widehat{CV}_i$\newline}

    \State{\textcolor{darkgreen}{\# Estimate variance for confidence interval.}}
    \State{$\widehat{\sigma}_i^2 = \left[\frac{1}{n_i}\widehat{\mathrm{Var}}_{j \in \mathcal{J}_i}(S_{ij}) - \left(\frac{1}{n_i} - \frac{1}{N}\right) \hat{\lambda}_i^2 \widehat{\mathrm{Var}}_{j \in [N]}(\hat{Y}_{ij})\right]^+$\newline}
\Return{$\left(\hat{\theta}^{\hat{\lambda}_i}_i, \left\{\hat{\theta}^{\hat{\lambda}_i}_i \pm z_{1-\frac{\alpha}{2}} \widehat{\sigma}_i \right\}\right)$}
\end{algorithmic}
\end{algorithm}
}
\end{figure}

\begin{figure}
\centering
\begin{algorithm}[H]
\caption{CollabEval for Model Comparisons (\texttt{EstimateDelta})}
\label{alg:collab_eval_delta}
{
\textbf{Definitions:} $\hat{Y}$ is the completed matrix from Algorithm~\ref{alg:xfold_complete}. $S$ is the incomplete matrix of true evaluation scores. $\Omega$ is the observation mask. $\alpha$ is the desired coverage of the CI for $\Delta_{ij}$.}\\
{
\textbf{Input:} $(i, j)$, the indices of the model pair.
}
\begin{algorithmic}[1]
    \State{\textcolor{darkgreen}{\# Compute overlap metrics.}}
    \State{$\mathcal{J}_i, \mathcal{J}_j = \{k \colon \Omega_{ik} = 1\}, \{k \colon \Omega_{jk} = 1\}$}
    \State{$n_i, n_j, n_{ij}, \gamma_{ij} = |\mathcal{J}_i|, |\mathcal{J}_j|, |\mathcal{J}_i \cap \mathcal{J}_j|, \frac{|\mathcal{J}_i \cap \mathcal{J}_j|}{|\mathcal{J}_i||\mathcal{J}_j|} - \frac{1}{N}$}
    
    \State{\textcolor{darkgreen}{\# Construct $\hat{Q}$ and $\hat{u}$ using empirical covariance estimates.}}
    \State{$\hat{Q} = \begin{pmatrix} \gamma_{ii} \widehat{\mathrm{Var}}_{k \in [N]}(\hat{Y}_{ik}) & -\gamma_{ij} \mathrm{Cov}_{k \in [N]}(\hat{Y}_{ik}, \hat{Y}_{jk}) \\ -\gamma_{ji} \mathrm{Cov}_{k \in [N]}(\hat{Y}_{jk}, \hat{Y}_{ik}) & \gamma_{jj} \widehat{\mathrm{Var}}_{k \in [N]}(\hat{Y}_{jk}) \end{pmatrix}$\newline}
    \State{$\hat{u}~= \begin{pmatrix} \gamma_{ii} \widehat{\mathrm{Cov}}_{k \in \mathcal{J}_i}(S_{ik}, \hat{Y}_{ik}) - \gamma_{ij} \widehat{\mathrm{Cov}}_{k \in \mathcal{J}_j}(\hat{Y}_{ik}, S_{jk}) \\ \gamma_{jj} \widehat{\mathrm{Cov}}_{k \in \mathcal{J}_j}(S_{jk}, \hat{Y}_{jk}) - \gamma_{ij} \widehat{\mathrm{Cov}}_{k \in \mathcal{J}_i}(\hat{Y}_{jk}, S_{ik}) \end{pmatrix}$\newline}

    \State{\textcolor{darkgreen}{\# Solve linear system. Note $\hat{\lambda} = (\hat{\lambda}_i, \hat\lambda_j)^T$ and $\dagger$ is the pseudo-inverse.}}
    \State{$\hat{\lambda} =  \hat{Q}^{\dagger} \hat{u}$\newline}

    \State{\textcolor{darkgreen}{\# Compute difference estimate from mean estimates using $\hat{\lambda}$ (same as Algorithm~\ref{alg:collab_eval_theta}).}}
    \State{$\hat{\Delta}^{\hat{\lambda}}_{ij} = \hat{\theta}^{\hat{\lambda}_i}_i - \hat{\theta}^{\hat{\lambda}_j}_j$\newline}
    
    \State{\textcolor{darkgreen}{\# Estimate variance for confidence interval (sample covariance evaluates to 0 if $|\mathcal{J}_i \cap \mathcal{J}_j| \le 1$).}}
    \State{$\widehat{\sigma}_{ij}^2 = \left[\frac{1}{n_i}\widehat{\mathrm{Var}}_{k \in \mathcal{J}_i}(S_{ik}) + \frac{1}{n_j}\widehat{\mathrm{Var}}_{k \in \mathcal{J}_j}(S_{jk}) - 2 \frac{n_{ij}}{n_in_j} \widehat{\mathrm{Cov}}_{k \in \mathcal{J}_i \cap \mathcal{J}_j}(S_{ik}, S_{jk}) - \hat{\lambda}^\top \hat{u}\right]^{+}$\newline}
\Return{$\left(\hat{\Delta}^{\hat{\lambda}}_{ij}, \left\{\hat{\Delta}^{\hat{\lambda}_{ij}} \pm z_{1-\frac{\alpha}{2}} \widehat{\sigma}_{ij} \right\}\right)$}
\end{algorithmic}
\end{algorithm}
\end{figure}

\clearpage
\subsection{IterativeSVD ablations}
\label{sec:isvd_ablation}

Table~\ref{tab:rank_ablation} provides an ablation for IterativeSVD (\Cref{alg:isvd}) across datasets at a fixed labeled fraction of 50\%, showing that increasing the max rank in $\mathcal{R}$ leads to better performance (it starts to plateau around rank 16). We progress the  rank sequence in powers of 2 from $(1)$, $(1,2)$, \dots, $(1,2,\dots, 32)$).

\begin{table}[ht]
\centering

\begin{tabular}{lcccccc}
\toprule
Task & rank=1 & rank=2 & rank=4 & rank=8 & rank=16 & rank=32 \\
\midrule
WMT & +8.4\% & +10.7\% & +11.4\% & +11.6\% & +12.4\% & +13.4\% \\
AQA NLI & +9.3\% & +11.2\% & +13.4\% & +16.8\% & +17.5\% & +17.9\% \\
MMLU & +5.0\% & +5.7\% & +8.2\% & +9.5\% & +10.6\% & +10.8\% \\
Alpaca Eval & +5.9\% & +8.1\% & +9.5\% & +10.8\% & +11.4\% & +12.0\% \\
SWE-bench & +9.3\% & +10.3\% & +10.3\% & +10.2\% & +10.7\% & +10.8\% \\
\midrule
Average & +7.6\% & +9.2\% & +10.5\% & +11.8\% & +12.5\% & +13.0\% \\
\bottomrule
\end{tabular}
\caption{IterativeSVD rank ablation at a fixed 50\% labeled fraction. Values represent the \% CI size reduction vs Classic (IID) (higher is better).}
\label{tab:rank_ablation}
\end{table}

Table~\ref{tab:fold_sweep} provides a similar ablation over the number of folds $K$, again at a fixed 50\% labeled fraction. We find that IterativeSVD is empirically not particularly sensitive to this choice---though, in theory, the quality of the matrix completion can improve with higher $K$ (more observed data per completion) which can narrow CI width. CI coverage remains unaffected, consistent with our theoretical results.

\begin{table}[ht]
\centering

\begin{tabular}{lcccc}
\toprule
Task & $K=2$ & $K=3$ & $K=5$ & $K=10$ \\
\midrule
WMT & 90.1\%, +12.7\% & 89.9\%, +13.2\% & 89.6\%, +12.4\% & 90.0\%, +12.0\% \\
AQA NLI & 88.1\%, +17.8\% & 89.7\%, +17.4\% & 91.6\%, +17.4\% & 89.1\%, +17.6\% \\
MMLU & 90.1\%, +10.6\% & 89.4\%, +10.6\% & 89.2\%, +10.6\% & 89.9\%, +10.6\% \\
Alpaca Eval & 89.6\%, +11.2\% & 90.5\%, +11.4\% & 89.5\%, +11.5\% & 90.2\%, +11.5\% \\
SWE-bench & 89.5\%, +10.6\% & 91.2\%, +10.7\% & 90.5\%, +10.6\% & 89.8\%, +10.7\% \\
\midrule
Average & 89.5\%, +12.6\% & 90.1\%, +12.7\% & 90.1\%, +12.5\% & 89.8\%, +12.5\% \\
\bottomrule
\end{tabular}
\caption{IterativeSVD fold ablation at a fixed 50\% labeled fraction. Values represent average coverage and \% CI size reduction vs Classic (IID), using default ranks $(1,2,4,8,16,16,16,16)$.}
\label{tab:fold_sweep}
\end{table}

\subsection{Comparison to other matrix completion options}
\label{sec:mc_ablation}

Table~\ref{tab:mc_ablation} compares CollabEval across four matrix completion methods at a fixed 50\% labeled data fraction: IterativeSVD, a neural network (NN), nuclear norm minimization, and a 2-parameter logistic (2PL) item response theory (IRT) model. The NN baseline uses a 1-layer, 64-dimensional multilayer perceptron applied to the vector of anchor model scores to predict target model scores, trained with either a regression or binary cross-entropy loss depending on the task. The 2PL IRT model applies only to the two binary classification datasets (MMLU and SWE-bench). 

All evaluated methods reduce confidence interval size relative to classical estimation. The NN yields the largest average reduction (+13.5\%) but exhibits higher variance across individual tasks at lower labeled data fractions. IterativeSVD provides a better tradeoff of computational speed, implementation simplicity, and consistent empirical performance. The statistical validity of CollabEval remains independent of the specific matrix completion algorithm used.

\begin{table}[ht]
\centering

\begin{tabular}{lcccc}
\toprule
Task & 2PL IRT & IterativeSVD & NN & Nuclear Norm \\
\midrule
WMT & N/A & +12.9\% & +12.6\% & +9.5\% \\
AQA & N/A & +17.3\% & +19.1\% & +15.4\% \\
MMLU & +1.7\% & +10.6\% & +15.0\% & +9.5\% \\
Alpaca Eval & N/A & +11.5\% & +9.5\% & +5.9\% \\
SWE-bench & +5.9\% & +10.5\% & +11.2\% & +10.6\% \\
\midrule
Average & +3.8\% & +12.5\% & +13.5\% & +10.2\% \\
\bottomrule
\end{tabular}
\caption{Matrix completion method ablation at a fixed 50\% labeled fraction. Values represent the \% reduction in pointwise $\theta$ CI size relative to the classic (IID) estimate.}
\label{tab:mc_ablation}
\end{table}
\section{Proofs}
\label{app:proofs}

\subsection{Proof of Theorem~\ref{thm:collab_clt}}


We proceed in two steps, following a similar analysis as \citet{zrnic2025cross}. First, we define an idealized estimator $\tilde{\theta}$ using the deterministic population-limit predictions $\hat{\bunderline{S}}_{ij}$, and prove its asymptotic normality via standard central limit theorem arguments. Then, we show that the empirical estimator $\hat{\theta}$ that uses cross-prediction is asymptotically equivalent.

\begin{lemma}
\label{lemm:idealized}
Define the idealized predictor:
\begin{align}
\tilde{\theta}_i &= \frac{1}{|\mathcal{J}_i|} \sum_{j \in \mathcal{J}_i} S_{ij} - \left(\frac{1}{|\mathcal{J}_i|} \sum_{j \in \mathcal{J}_i} \hat{\bunderline{S}}_{ij} - \frac{1}{N}\sum_{j=1}^N \hat{\bunderline{S}}_{ij}\right).
\end{align}
Then, under the conditions of Theorem~\ref{thm:collab_clt}, $\sqrt{N}(\tilde{\theta} - \theta) \xrightarrow{d} \mathcal{N}(0, \Sigma)$, where
\begin{gather}
\Sigma_{ij} = \frac{p_{ij}}{p_i p_j} \mathrm{Cov}\left(S_i, S_j\right) + \left(\frac{p_{ij}}{p_i p_j} - 1\right) \left[ \mathrm{Cov}\left(\hat{\bunderline{S}}_i, \hat{\bunderline{S}}_j\right) - \mathrm{Cov}\left(\hat{\bunderline{S}}_i, S_j\right) - \mathrm{Cov}\left(S_i, \hat{\bunderline{S}}_j\right) \right].
\end{gather}
\end{lemma}

\begin{proof}
We begin by rewriting $\tilde{\theta}$ as
$$\tilde{\theta} = \frac{1}{N} \sum_{j=1}^N Z_j,$$
where $Z_{j}$ is an $M$-dimensional  random vector  with
$$Z_{j}(i) = \hat{\bunderline{S}}_{ij} + \frac{N}{n_i} \Omega_{ij} (S_{ij} - \hat{\bunderline{S}}_{ij}).$$

Because the $N$ evaluation prompts are drawn i.i.d. and the population-limit prediction $\hat{\bunderline{S}}_{ij} = h_i(X_j)$ where $X_j = \{S_{kj} \colon \Omega_{kj} = 1, k \in \mathcal{A}\}$ is a deterministic function of only the observed anchor model scores for the $j$-th prompt (Assumption~\ref{ass:stable}), the vectors $Z_1, Z_2, \dots, Z_N$ are conditionally independent given  the observation mask $\Omega$, and the observed anchor model scores.

To establish asymptotic normality, we view $\{Z_j\}_{j=1}^N$ as a triangular array, since the scaling coefficients $\frac{N}{n_i}$ depend on $N$. By assumption, the sampling budgets scale deterministically as $n_i/N \to p_i \in (0, 1]$, meaning $\frac{N}{n_i}$ are bounded by a constant for sufficiently large $N$. Because $(S_{ij}, \hat{\bunderline{S}}_{ij})$ have finite covariance by assumption, and $Z_j$ is a linear combination of these variables with bounded coefficients and $\Omega_{ij} \in \{0, 1\}$, the array is bounded by a square-integrable distribution; this  implies the Lindeberg condition.\looseness=-1

Applying the Multivariate Lindeberg-Feller CLT to this array, we have:
$$\sqrt{N}(\tilde{\theta} - \mathbb{E}[\tilde{\theta} \mid \Omega]) \mid \Omega \xrightarrow{d} \mathcal{N}(0, \Sigma_N(\Omega))$$
for some conditional covariance matrix $\Sigma_N(\Omega)$. Furthermore, taking the conditional expectation yields $\mathbb{E}[Z_{ij} \mid \Omega] = \mathbb{E}[\hat{\underline{S}}_{ij}] + \frac{N}{n_i} \Omega_{ij} \mathbb{E}[S_{ij} - \hat{\underline{S}}_{ij}]$. Summing over $j$ and substituting $\sum_{j=1}^N \Omega_{ij} = n_i$ deterministically gives $\mathbb{E}[\tilde{\theta} \mid \Omega] = \theta$. Therefore, we can substitute:
$$\sqrt{N}(\tilde{\theta} - \theta) \mid \Omega \xrightarrow{d} \mathcal{N}(0, \Sigma_N(\Omega)).$$

We now calculate $\Sigma_N(\Omega)$. By the conditional independence of $Z_j$:
$$ \Sigma_N(\Omega) = \mathrm{Var}(\sqrt{N}\tilde{\theta} \mid \Omega) = \frac{1}{N} \sum_{j=1}^N \mathrm{Var}(Z_j \mid \Omega).$$
Looking at the covariance between the $i$-th and $k$-th components of $Z_j$, we have
$$ \text{Cov}(Z_{ij}, Z_{kj} \mid \Omega) = \text{Cov}\left( \hat{\bunderline{S}}_{ij} + \frac{N}{n_i}\Omega_{ij}(S_{ij} - \hat{\bunderline{S}}_{ij}), \hat{\bunderline{S}}_{kj} + \frac{N}{n_k}\Omega_{kj}(S_{kj} - \hat{\bunderline{S}}_{kj}) \right). $$
Expanding this via bilinearity of covariance  and summing over $j$ yields:
\begin{align}
[\Sigma_N(\Omega)]_{ik} &= \frac{1}{N} \sum_{j=1}^N \Bigg [\mathrm{Cov}(\hat{\bunderline{S}}_{ij}, \hat{\bunderline{S}}_{kj}) + \frac{N}{n_k}\Omega_{kj} \mathrm{Cov}(\hat{\bunderline{S}}_{ij}, S_{kj} - \hat{\bunderline{S}}_{kj}) \\
&\quad\quad\quad\quad\quad+ \frac{N}{n_i}\Omega_{ij} \mathrm{Cov}(S_{ij} - \hat{\bunderline{S}}_{ij}, \hat{\bunderline{S}}_{kj}) + \frac{N^2}{n_i n_k}\Omega_{ij}\Omega_{kj} \mathrm{Cov}(S_{ij} - \hat{\bunderline{S}}_{ij}, S_{kj} - \hat{\bunderline{S}}_{kj})\Bigg].
\end{align}

Because the underlying prompts and prompt features ($X_j$, that $\hat{\bunderline{S}}_{ij}$ are fixed functions of) are i.i.d., the population covariances do not depend on $j$. Using the identities $\sum_{j=1}^N \Omega_{ij} = n_i$, $\sum_{j=1}^N \Omega_{kj} = n_k$, and $\sum_{j=1}^N \Omega_{ij}\Omega_{kj} = n_{ik}$ from our sampling design, the equation simplifies to:
\begin{align*}
[\Sigma_N(\Omega)]_{ik}
&= \mathrm{Cov}(\hat{\bunderline{S}}_i, \hat{\bunderline{S}}_k) + \left(\frac{1}{N}\frac{N}{n_k}n_k\right) \mathrm{Cov}(\hat{\bunderline{S}}_i, S_k - \hat{\bunderline{S}}_k) + \left(\frac{1}{N}\frac{N}{n_i}n_i\right) \text{Cov}(S_i - \hat{\bunderline{S}}_i, \hat{\bunderline{S}}_k) \\
&\quad + \left(\frac{1}{N}\frac{N^2}{n_i n_k}n_{ik}\right) \mathrm{Cov}(S_i - \hat{\bunderline{S}}_i, S_k - \hat{\bunderline{S}}_k) \\
&= \mathrm{Cov}(\hat{\bunderline{S}}_i, \hat{\bunderline{S}}_k) + \mathrm{Cov}(\hat{\bunderline{S}}_i, S_k) - \mathrm{Cov}(\hat{\bunderline{S}}_i, \hat{\bunderline{S}}_k) + \mathrm{Cov}(S_i, \hat{\bunderline{S}}_k) - \text{Cov}(\hat{\bunderline{S}}_i, \hat{\bunderline{S}}_k) \\
&\quad + \frac{N n_{ik}}{n_i n_k} \left[ \mathrm{Cov}(S_i, S_k) - \mathrm{Cov}(S_i, \hat{\bunderline{S}}_k) - \mathrm{Cov}(\hat{\bunderline{S}}_i, S_k) + \mathrm{Cov}(\hat{\bunderline{S}}_i, \hat{\bunderline{S}}_k) \right].
\end{align*}

Letting $\gamma_{ik}^N = \frac{N n_{ik}}{n_i n_k}$ and grouping terms, we finally arrive at the following conditional covariance matrix for our triangular array:
\begin{align*}
\Sigma_N(\Omega)_{ik} &= \gamma_{ik}^N \mathrm{Cov}(S_i, S_k) + (\gamma_{ik}^N - 1) \big[ \mathrm{Cov}(\hat{\bunderline{S}}_i, \hat{\bunderline{S}}_k) - \mathrm{Cov}(\hat{\bunderline{S}}_i, S_k) - \mathrm{Cov}(S_i, \hat{\bunderline{S}}_k) \big].
\end{align*}

By our sampling assumptions, as $N \to \infty$, the deterministic budgets scale as $n_i/N \to p_i$ and $n_k/N \to p_k$, while the random overlap fraction converges in probability: $n_{ik}/N \xrightarrow{p} p_{ik}$. By the Continuous Mapping Theorem, the coefficient $\gamma_{ik}^N$ converges in probability:
$$ \gamma_{ik}^N = \frac{n_{ik}/N}{(n_i/N)(n_k/N)} \xrightarrow{p} \frac{p_{ik}}{p_i p_k}.$$
Consequently, $\Sigma_N(\Omega)_{ik}$ converges in probability to a constant limit matrix $\Sigma$:
$$ \Sigma_N(\Omega) \xrightarrow{p} \Sigma, $$
where the entries of $\Sigma$ are exactly:
$$ \Sigma_{ik} = \frac{p_{ik}}{p_i p_k} \mathrm{Cov}(S_i, S_k) + \left(\frac{p_{ik}}{p_i p_k} - 1\right) \big[ \mathrm{Cov}(\hat{\bunderline{S}}_i, \hat{\bunderline{S}}_k) - \mathrm{Cov}(\hat{\bunderline{S}}_i, S_k) - \mathrm{Cov}(S_i, \hat{\bunderline{S}}_k) \big]. $$
Since $\sqrt{N}(\tilde{\theta} - \theta) \mid \Omega \xrightarrow{d} \mathcal{N}(0, \Sigma_N(\Omega))$ and $\Sigma_N(\Omega) \xrightarrow{p} \Sigma$, we can apply Slutsky's Theorem to conclude unconditionally that $\sqrt{N}(\tilde{\theta} - \theta) \xrightarrow{d} \mathcal{N}(0, \Sigma)$.
\end{proof}

We now are ready to state our second lemma.

\begin{lemma}
\label{lemma:convergence}
Under the conditions of Theorem~\ref{thm:collab_clt}, as $N \to \infty$, we have $\sqrt{N}(\hat{\theta} - \tilde{\theta}) \xrightarrow{p} 0$.
\end{lemma}
\begin{proof}
We start by writing
\begin{align}
\sqrt{N}(\hat{\theta}_i - \tilde{\theta}_i) = \sqrt{N}\left(\frac{1}{N} \sum_{j=1}^N (\hat{Y}_{ij} - \hat{S}_{ij}) + \frac{N}{n_i} \Omega_{ij}(\hat{S}_{ij} - \hat{Y}_{ij})\right).
\end{align}
Let $w_{ij} = \sqrt{N} \left( \frac{1}{N} - \frac{\Omega_{ij}}{n_i} \right)$  and define $V_{N,k}$ as
\begin{align}
 V_{N,k} &= \sum_{j \in \mathcal{J}_i^{(k)}} w_{ij} (\hat{S}_{ij}^{(k)} - \hat{\bunderline{S}}_{ij}) + \frac{1}{K} \sum_{j \in \mathcal{G}_i} \frac{1}{\sqrt{N}} (\hat{S}_{ij}^{(k)} - \hat{\bunderline{S}}_{ij}) \\
 &= \sum_{j \in \mathcal{J}_i^{(k)}} w_{ij} \Gamma_{ij} + \frac{1}{K} \sum_{j \in \mathcal{G}_i} \frac{1}{\sqrt{N}} \Gamma_{ij}
\end{align}
where $\Gamma_{ij} = \hat{S}_{ij}^{(k)} - \hat{\bunderline{S}}_{ij}$, $\mathcal{J}_i^{(k)}$ are the observed indices in fold $k$, and $\mathcal{G}_i = \{N\} \setminus \mathcal{J}_i$ are the unobserved indices. Let $V_N = \sum_{k=1}^K V_{N,k}$. Note that $V_N = \sqrt{N}(\hat{\theta}_i - \tilde{\theta}_i)$.

Using Fact 1 from ~\citet{zrnic2025cross}, $V_N \xrightarrow{p} 0$ iff $\mathbb{E}[\min(1, |V_N|)] \to 0$. We now apply the same analysis as in Lemma A.1 and A.2 of \citet{zrnic2025cross}. Let $\psi(x) = \min(1, x)$. Since $\psi$ is subadditive and $V_N = \sum_{k=1}^K V_{N,k}$,
$$\mathbb{E}[\psi(|V_N|)] \le \sum_{k=1}^K \mathbb{E}[\psi(|V_{N,k}|)]$$
Since $\psi$ is concave, we condition on the mask $\Omega^{(k)}$ used for the $k$-th fold, and apply Jensen's:
\begin{align}
\mathbb{E}[\psi(|V_{N,k}|)] &= \mathbb{E}[\mathbb{E}[\psi(|V_{N,k}|) \mid \Omega^{(k)}]] \\
&\le \mathbb{E} \left[ \psi \left( \mathbb{E}[|V_{N,k}| \mid \Omega^{(k)}] \right) \right] \\
&\le \mathbb{E} \left[ \psi \left( \sqrt{\text{Var}(V_{N,k} \mid \Omega^{(k)})} \right) \right].
\end{align}
Since the indices $j \in \mathcal{J}_i^{(k)}$ are held-out from the training of $\hat{S}^{(k)}$, the terms in the $V_{N,k}$ are conditionally independent given $\Omega^{(k)}$, so we can write:
$$\text{Var}(V_{N,k} \mid \Omega^{(k)}) = \sum_{j \in \mathcal{J}_i^{(k)}} w_{ij}^2 \text{Var}(\Gamma_{ij} \mid \Omega^{(k)}) + \frac{1}{K^2 N} \sum_{j \in \mathcal{G}_i} \text{Var}(\Gamma_{ij} \mid \Omega^{(k)}).$$
Note that for $j \in \mathcal{J}_i^{(k)}$, the weights $w_{ij}^2 = N(\frac{1}{N} - \frac{1}{n_i})^2 = \frac{1}{N}(1 - \frac{N}{n_i})^2$ are $O(1/N)$, since $\frac{N}{n_i} \rightarrow 1 / p_i$. Since the number of terms in each sum is $O(N)$, the total variance for fold $k$ is bounded by a constant multiple of the average stability error:
$$\text{Var}(V_{N,k} \mid \Omega^{(k)}) \le C \cdot \text{Var}(\hat{S}_i^{(k)} - \hat{\bunderline{S}}_i \mid \Omega^{(k)})$$for some constant $C$, depending on the sampling fraction $p_i$. 

Substituting this back into the expectation of the concave function $\psi$:
$$\sum_{k=1}^K \mathbb{E}[\psi(|V_{N,k}|)] \le \sum_{k=1}^K \mathbb{E} \left[ \psi \left( \sqrt{C \cdot \text{Var}(\hat{S}_i^{(k)} - \hat{\bunderline{S}}_i \mid \Omega^{(k)})} \right) \right].$$

By Assumption~\ref{ass:stable}, $\sqrt{K \text{Var}(\hat{S}_i^{(1)} - \hat{\bunderline{S}}_i \mid \Omega^{(1)})} \xrightarrow{L_1} 0$  as $N \to \infty$. Since $\psi(x) = \min(1, x) \le x$ for $x \ge 0$, and the expectation of the square root vanishes, it follows that:
$$\mathbb{E}[\psi(|V_N|)] \le \sum_{k=1}^K \mathbb{E}[\psi(|V_{N,k}|)] \to 0.$$
By Fact 1, $\mathbb{E}[\min(1, |V_N|)] \to 0$ implies $V_N \xrightarrow{p} 0$. Thus, $\sqrt{N}(\hat{\theta}_i - \tilde{\theta}_i) \xrightarrow{p} 0$ for each $i \in [M]$. Since $M$ is finite, the vector convergence $\sqrt{N}(\hat{\theta} - \tilde{\theta}) \xrightarrow{p} {0}$ holds simultaneously.
\end{proof}

We can now prove Theorem~\ref{thm:collab_clt}.

\begin{proof}
Applying Slutsky's Theorem with the results of Lemmas~\ref{lemm:idealized} and \ref{lemma:convergence} yields:
$$\sqrt{N}(\hat{\theta} - \theta) \xrightarrow{d} \mathcal{N}(0, \Sigma).$$
\end{proof}





\subsection{Proof of Corollary~\ref{cor:linear_functions}}
\begin{proof}
Follows directly from Theorem~\ref{thm:collab_clt} and the fact that
\begin{equation}
X \xrightarrow{d} \mathcal{N}(\mu, \Sigma) \Longrightarrow a^\top X \xrightarrow{d} \mathcal{N}(a^\top \mu, a^\top \Sigma a).
\end{equation}
for constant vector $a \in \mathbb{R}^M$.
\end{proof}

\subsection{Proof of Corollary~\ref{cor:collab_ci}}
\begin{proof}
This follows from Theorem~\ref{thm:collab_clt} and Corollary~\ref{cor:linear_functions}, together with Slutsky's Theorem.
\end{proof}

\subsection{Proof of Theorem~\ref{thm:empirical_cov}}
\begin{proof}
We prove this by establishing the convergence in probability of the individual components of $\hat{\Sigma}_{ij}$ to constants, and then applying Slutsky's Theorem.

By Assumption \ref{ass:stable}, $\sqrt{K \mathrm{Var}( \hat{S}^{(k)}_{ij} - \hat{\bunderline{S}}_{ij} \mid \Omega^{(k)} )} \xrightarrow{L_1} 0$. Because $\hat{\bunderline{S}}_{ij}$ is the deterministic population limit, the conditional bias $\mathbb{E}[\hat{S}^{(k)}_{ij} - \hat{\bunderline{S}}_{ij} \mid \Omega^{(k)}]$ also vanishes. This combined with strictly bounded evaluation scores guarantees unconditional mean-square ($L_2$) convergence, which implies the convergence in probability of the individual fold predictions: $\hat{S}^{(k)}_{ij} \xrightarrow{p} \hat{\bunderline{S}}_{ij}$. Since the final prediction $\hat{Y}_{ij}$ is constructed via finite sums of these fold predictions, the Continuous Mapping Theorem yields $\hat{Y}_{ij} \xrightarrow{p} \hat{\bunderline{S}}_{ij}$.

Finally, all terms are bounded. The empirical covariance of bounded i.i.d. variables is a consistent estimator of the true covariance; substituting the converged predictions $\hat{Y}_{ij}$ preserves this. Since the deterministic sampling weights also converge to constants, Slutsky's Theorem gives $\hat{\Sigma}_{ij} \xrightarrow{p} \Sigma_{ij}$.
\end{proof}

\subsection{Proof of Proposition~\ref{cor:optimal_lambda}}
\label{sec:proof_optimal_lambda}
\begin{proof}
We begin with $\hat{\theta}_i^{\lambda_i}$. From Corollary~\ref{cor:linear_functions}, the asymptotic variance as a function of $\lambda$ is 
$$V_{\hat{\theta}_i}(\lambda_i) = \frac{1}{p_i}\mathrm{Var}(S_i) + \left(\frac{1}{p_i} - 1\right)\left[\lambda_i^2 \mathrm{Var}(\hat{\bunderline{S}}_i) - 2\lambda_i \mathrm{Cov}(S_i, \hat{\bunderline{S}}_i)\right].$$
Differentiating with respect to $\lambda_i$ and setting it to zero gives the optimal weight:
$$\lambda_i^* = \frac{\mathrm{Cov}(S_i, \hat{\bunderline{S}}_i)}{\mathrm{Var}(\hat{\bunderline{S}}_i)}.$$

We now derive the optimal weight for $\hat{\Delta}_{ij}^{\lambda_{ij}}$, again using Corollary~\ref{cor:linear_functions}. Let $\gamma_{ij} = \frac{p_{ij}}{p_ip_j} - 1$.

Proceeding as before:
\begin{align}
    V_{\hat{\Delta}_{ij}}(\lambda_i, \lambda_j) &= \frac{1}{p_i}\mathrm{Var}(S_i) + \gamma_{ii}\left[\lambda_i^2 \mathrm{Var}(\hat{\bunderline{S}}_i) - 2\lambda_i \mathrm{Cov}(S_i, \hat{\bunderline{S}}_i)\right] \\
    &+ \frac{1}{p_j}\mathrm{Var}(S_j) + \gamma_{ij}\left[\lambda_j^2 \mathrm{Var}(\hat{\bunderline{S}}_j) - 2\lambda_j \mathrm{Cov}(S_j, \hat{\bunderline{S}}_j)\right] \\
     &-2\left[\frac{p_{ij}}{p_ip_j}\mathrm{Cov}(S_i, S_j) + \gamma_{ij} \left[\lambda_i\lambda_j \mathrm{Cov}( \hat{\bunderline{S}}_i, \hat{\bunderline{S}}_j) - \lambda_i \mathrm{Cov}(\hat{\bunderline{S}}_i, S_j) -\lambda_j  \mathrm{Cov}(S_i, \hat{\bunderline{S}}_j)\right]\right]
\end{align}
Consolidating terms, and writing $\mathrm{Var}(\hat{\Delta}_{ij}^\mathrm{classic}) = \frac{1}{p_i}\mathrm{Var}(S_i) + \frac{1}{p_j}\mathrm{Var}(S_j) - 2\frac{p_{ij}}{p_ip_j}\mathrm{Cov}(S_i, S_j)$,
\begin{align}
  V_{\hat{\Delta}_{ij}}(\lambda_i, \lambda_j)   &= \mathrm{Var}(\hat{\Delta}_{ij}^\mathrm{classic}) \\
  &+ \underbrace{\lambda_i^2 \gamma_{ii} \mathrm{Var}(\hat{\bunderline{S}}_i) + \lambda_j^2 \gamma_{jj} \mathrm{Var}(\hat{\bunderline{S}}_j) - 2 \lambda_i \lambda_j \gamma_{ij} \mathrm{Cov}(\hat{\bunderline{S}}_i, \hat{\bunderline{S}}_j)}_{\lambda^\top Q \lambda} \\ 
     &- \underbrace{2 [\lambda_i (\gamma_{ii} \mathrm{Cov}(S_i, \hat{\bunderline{S}}_i) - \gamma_{ij} \mathrm{Cov}(\hat{\bunderline{S}}_i, S_j)) + \lambda_j (\gamma_{jj} \mathrm{Cov}(S_j, \hat{\bunderline{S}}_j) - \gamma_{ij} \mathrm{Cov}(S_i, \hat{\bunderline{S}}_j))]}_{2 \lambda^\top u}
\end{align}
where $Q$ and $u$ are defined as in Proposition~\ref{cor:optimal_lambda}.

To minimize $V(\lambda_i, \lambda_j)$, we again take the gradient with respect to $\lambda$ and set it to zero:
$$\nabla_\lambda V(\lambda) = 2Q\lambda - 2u = 0,$$
which gives $\lambda^* = Q^{-1}u$.
\end{proof}

\subsection{Proof of Proposition~\ref{prop:power_tuning_distribution}}
\begin{proof}
Slutsky's Theorem preserves convergence in distribution. The rewritten form of the asymptotic variance can then be seen from substituting  $\lambda_i^*\hat{\bunderline{S}_i}$ directly into Equation~\eqref{eq:limiting_cov} and simplifying. Following the derivations in the proof of Proposition~\ref{cor:optimal_lambda} in Section~\ref{sec:proof_optimal_lambda},
\begin{align}
V_{\hat{\theta}_i}(\lambda_i^*) &= \frac{1}{p_i}\mathrm{Var}(S_i) + \left(\frac{1}{p_i} - 1\right)\left[\frac{\mathrm{Cov}(S_i, \hat{\bunderline{S}}_i)^2}{\mathrm{Var}(\hat{\bunderline{S}}_i)^2} \mathrm{Var}(\hat{\bunderline{S}}_i) - 2\frac{\mathrm{Cov}(S_i, \hat{\bunderline{S}}_i)}{\mathrm{Var}(\hat{\bunderline{S}}_i)}\mathrm{Cov}(S_i, \hat{\bunderline{S}}_i)\right] \\
&=\mathrm{Var}(\hat{\theta}_i^\mathrm{classic}) - \left(\frac{1}{p_i} - 1\right)\frac{\mathrm{Cov}(S_i, \hat{\bunderline{S}}_i)^2}{\mathrm{Var}(\hat{\bunderline{S}}_i)},
\end{align}
and
\begin{align}
V_{\hat{\Delta}_{ij}}(\lambda^*_i, \lambda^*_j) &=\mathrm{Var}(\hat{\Delta}_{ij}^\mathrm{classic}) + (\lambda^*)^\top Q \lambda^* - 2(\lambda^*)^\top u \\
&= \mathrm{Var}(\hat{\Delta}_{ij}^\mathrm{classic}) + (Q^{-1}u)^\top Q Q^{-1}u - 2(Q^{-1}u)^\top u \\
&= \mathrm{Var}(\hat{\Delta}_{ij}^\mathrm{classic}) - u^\top Q^{-1} u,
\end{align}
since $Q^{-1}$ is symmetric.
\end{proof}
\end{document}